\documentclass[english]{article}
\usepackage{geometry}

\usepackage{graphicx} 
\usepackage[skins,theorems]{tcolorbox}
\tcbset{highlight math style={enhanced,colframe=red,colback=white}}

\usepackage{graphics}

\usepackage{natbib}
\usepackage{amsmath, amssymb,float,booktabs,mathabx}

\usepackage{enumitem}
\usepackage[colorlinks=true,allcolors=blue]{hyperref}

\usepackage{amsmath,bbm,bm}
\usepackage{amsfonts}
\usepackage{amsthm}
\usepackage{mathtools}

\usepackage{algorithm}
\usepackage{algpascal}
\usepackage{algorithmicx}
\usepackage{algpseudocode}
\usepackage{tabularx}
\algdef{SE}[SUBALG]{Indent}{EndIndent}{}{\algorithmicend\ }%
\algtext*{Indent}
\algtext*{EndIndent}

\newtheorem{thm}{Theorem}
\newtheorem{lem}{Lemma}

\newtheorem{cor}{Corollary}

\newtheorem{aspt}{Assumption}

\newtheorem{defn}{Definition}

\usepackage{xcolor}
\newcount\comments  
\newcommand{\genComment}[2]{\ifnum\comments=1{\textcolor{#1}{\textsf{\footnotesize #2}}}\fi}

\def\eqref#1{equation~\ref{#1}}

\def\1{\bm{1}}

\DeclareMathAlphabet{\mathsfit}{\encodingdefault}{\sfdefault}{m}{sl}
\SetMathAlphabet{\mathsfit}{bold}{\encodingdefault}{\sfdefault}{bx}{n}

\def\gA{{\mathcal{A}}}

\def\gD{{\mathcal{D}}}

\def\gH{{\mathcal{H}}}
\def\gI{{\mathcal{I}}}
\def\gJ{{\mathcal{J}}}

\def\gL{{\mathcal{L}}}
\def\gM{{\mathcal{M}}}

\def\gP{{\mathcal{P}}}

\def\gS{{\mathcal{S}}}

\def\gX{{\mathcal{X}}}

\def\prob{{\mathbb{P}}}

\def\R{{\mathbb{R}}}

\newcommand{\E}{\mathbb{E}}

\DeclareMathOperator*{\argmax}{arg\,max}

\newcommand{\off}{\operatorname{off}}
\newcommand{\on}{\operatorname{on}}

\definecolor{ballblue}{rgb}{0.13, 0.67, 0.8}
\definecolor{darkred}{RGB}{139,0,0}

\allowdisplaybreaks

\title{Hybrid Reinforcement Learning Breaks Sample Size Barriers \\ in Linear MDPs} 

\usepackage{authblk}

\author{Kevin Tan}
\author{Wei Fan}
\author{Yuting Wei}
\affil{
  Department of Statistics and Data Science\\
  The Wharton School, University of Pennsylvania}
\date{}
\begin{document}

\maketitle

\begin{abstract}
    Hybrid Reinforcement Learning (RL), where an agent learns from both an offline dataset and online explorations in an unknown environment, has garnered significant recent interest. 
    A crucial question posed by \cite{xie2022policy} is whether hybrid RL can improve upon the existing lower bounds established in purely offline and purely online RL without relying on the single-policy concentrability assumption. 
    While \cite{li2023reward} provided an affirmative answer to this question in the tabular PAC RL case, the question remains unsettled for both the regret-minimizing RL case and the non-tabular case. 
    In this work,
    building upon recent advancements in offline RL and reward-agnostic exploration, we develop computationally efficient algorithms for both PAC and regret-minimizing RL with linear function approximation, without single-policy concentrability. 
    We demonstrate that these algorithms achieve sharper error or regret bounds that are no worse than, and can improve on, the optimal sample complexity in offline RL (the first algorithm, for PAC RL) and online RL (the second algorithm, for regret-minimizing RL) in linear Markov decision processes (MDPs), regardless of the quality of the behavior policy. 
    To our knowledge, this work establishes the tightest theoretical guarantees currently available for hybrid RL in linear MDPs.
\end{abstract}

\tableofcontents

\section{Introduction}

Reinforcement learning (RL) holds great promise in attaining reliable decision-making in adaptive environments for a broad range of modern applications. 
In these applications, typical RL algorithms often require an enormous number of training samples in order to reach the desired level of accuracy. 
This has motivated a line of recent efforts to study the sample efficiency of RL algorithms. 
There are two mainstream paradigms of RL, distinguished by how samples are collected: online RL and offline RL. 
In the setting of online RL, an agent learns in a real-time manner, exploring the environment to maximize her cumulative rewards by executing a sequence of adaptively chosen policies 
(e.g.~\cite{azar2017minimax,kearns2002near,jin2018q,sutton1998,zhang2023settling}).
These online RL algorithms often suffer from insufficient use of data samples due to a lack of a reference policy at the initial stage of the learning process. 
Whereas, in offline RL, an agent has only access to a pre-collected dataset, and tries to figure out how to perform well in a different environment without ever experiencing it 
(e.g.~\cite{levine2020offline,lange2012batch,jin2021pessimism,xie2022policy,li2024settling}). 
Offline methods therefore often impose stringent requirements on the quality of the pre-collected data.

To address limitations from both cases, the setting of hybrid RL \citep{xie2022policy, song2023hybrid} has recently received considerable attention from both theoretical and practical perspectives (see, e.g.~\cite{vecerik2017leveraging,nair2020awac,song2023hybrid, nakamoto2023calql, wagenmaker2023leveraging, li2023reward, ball2023efficient,zhou2023offline, amortila2024harnessing, tan2024natural, kausik2024leveragingofflinedatalinear} and references therein). 
In hybrid RL, an agent learns from a combination of both offline and online data, extracting information from offline data to enhance online exploration. 
The theoretical guarantees of hybrid RL algorithms can be categorized based on the following criteria:  
(1) the degree of function approximation considered, 
(2) the level of coverage required by the behavior policy, 
(3) whether it achieves an improvement over the \textit{minimax lower bounds} for online-only and offline-only learning, 
and (4) whether they attempt to perform regret minimization or to learn an $\epsilon$-optimal policy (often referred to as PAC learning).
We elaborate below, and summarize the prior art in Table \ref{tab:prev-lit}.
\begin{table}[htbp!]
\centering
\scalebox{0.85}{
\begin{tabular}{ccccc}
\toprule  Paper & Function Type & Concentrability? & Improvement? & Regret or PAC? \\
\toprule 
\cite{song2023hybrid}& General & Required & No & Regret\\
\cite{nakamoto2023calql}\\
\hline
\cite{tan2024natural}& General & Not Required & No & Regret\\
\cite{amortila2024harnessing}  
\\
\hline
\cite{wagenmaker2023leveraging}
& Linear & Not Required & No & PAC \\
\hline
\cite{li2023reward} & Tabular & Not Required & Yes & PAC
\\
\hline
\textcolor{blue}{This work} & Linear & Not Required & Yes & Regret, PAC\\
\toprule 
\end{tabular}
}
\caption{Comparison of our contributions to previous work in hybrid RL.}
\label{tab:prev-lit}
\end{table}


While most of the prior literature \citep{song2023hybrid, nakamoto2023calql, zhou2023offline, tan2024natural, amortila2024harnessing} have explored general function approximation in hybrid RL, they either require stringent concentrability assumptions on the quality of the behavior policy, or fail to obtain tight theoretical guarantees. 
Under such single-policy concentrability assumptions (explained below), it has been shown in \cite{xie2022policy} that the optimal policy learning algorithm is either a purely offline reduction or a purely online RL algorithm if the agent can choose the ratio of offline to online samples, rendering the benefits of hybrid RL questionable. 
In scenarios where this assumption is not satisfied \citep{li2023reward, wagenmaker2023leveraging, tan2024natural, amortila2024harnessing},  \cite{li2023reward} obtained theoretical guarantees for PAC RL that improve over lower bounds established for offline-only and online-only RL. However, the work of \cite{li2023reward} remains restricted to the tabular case with finite number of states and actions. 

To further explore the efficacy of hybrid RL, this paper focuses on obtaining sharper theoretical guarantees in the setting of linear function approximation, specifically in the case of \textit{linear MDPs}. First proposed in \cite{yang2019sampleoptimal, jin2019provably}, 
the linear MDP setting parameterizes the transition probability matrix and reward function by linear functions of known features. It has since been extensively studied due to its benefits in dimension reduction and mathematical traceability in both the online and offline settings (see, e.g.~\cite{yang2019sampleoptimal,qiao2022near,du2019good, li2021sample, jin2019provably, zanette2021provable,yin2022nearoptimal,xiong2023nearly, he2023nearly, hu2023nearlynot, min2021variance, duan2020minimaxoptimal}).
Despite these efforts, hybrid RL algorithms for linear MDPs \citep{wagenmaker2023leveraging, song2023hybrid, nakamoto2023calql, amortila2024harnessing, tan2024natural} have suboptimal worst-case guarantees (Table \ref{tab:best-bounds}), which raises the question:
\begin{center}
\emph{Is it possible to develop sample efficient RL algorithms in the setting of hybrid RL that are provably better than online-only and offline-only algorithms for linear MDPs?}
\end{center}

\subsection{Hybrid RL: two approaches}

Before answering the question above, we begin by introducing two types of approaches that are widely-adopted in hybrid RL. 

\paragraph{The offline-to-online approach:} Most of the current literature (e.g. \cite{song2023hybrid, nakamoto2023calql, amortila2024harnessing, tan2024natural}) adopts the approach of initializing the online dataset with offline samples, in order to perform regret-minimizing online RL. We shall refer to this as the \textit{offline-to-online} approach. This method is simple and natural, offering several additional benefits. 
For instance, the algorithm optimizes the reward received during each online episode, making it suitable when it is crucial for the agent to perform well on average during its online exploration.


\paragraph{The online-to-offline approach:} 
However, if our goal is to output a near-optimal policy, especially in real-world situations such as medical treatment and defense-related applications, it is not ideal and sometimes even unethical to provide a randomized policy with guarantees that hold only on average.
Recently, \cite{wagenmaker2023leveraging} and \cite{li2023reward} propose using reward-agnostic online exploration to explore the parts of the state space that are not well-covered by the behavior policy, thereby constructing a dataset that is especially amenable for offline RL. We refer to this as the \textit{online-to-offline approach}. This method allows leveraging the sharp performance guarantees of offline RL when the single-policy concentrability coefficient is low. While this approach does not optimize the ``true reward'' during online exploration, and  is therefore incompatible with the regret minimization framework, it avoids the need to deploy mixed policies to achieve a PAC bound, allowing for the deployment of fixed, and thus more interpretable, policies.




\subsection{Our contributions} 



In view of these two types of hybrid RL algorithms, focusing on the setting of linear MDPs, we answer the aforementioned question in the affirmative. 
We summarize our main contributions below. 

\begin{itemize}

    \item We propose an online-to-offline algorithm called \emph{Reward-Agnostic Pessimistic PAC Exploration-initialized Learning (RAPPEL)} in Algorithm \ref{alg:rappel}. This algorithm employs reward-agnostic online exploration to enhance the offline dataset, followed by a pessimistic offline RL algorithm to learn an optimal policy. We show that the sample complexity of Algorithm \ref{alg:rappel} significantly improves upon the only dedicated hybrid RL algorithm for linear MDPs \citep{wagenmaker2023leveraging} by a factor of at least $H^3$ (with $H$ the time horizon). 
    Additionally, this result demonstrates that hybrid RL can perform no worse than the offline-only minimax-optimal error bound from \cite{xiong2023nearly}, with the potential of significant gains if one has access to a large number of online samples. This is also the first work to explore the online-to-offline approach in linear MDPs.

    \item In addition, we propose an offline-to-online method called \emph{Hybrid Regression for Upper-Confidence Reinforcement Learning (HYRULE)} in Algorithm \ref{alg:hyrule}, where one warm-starts an online RL algorithm with parameters estimated from offline data.
    In addition to improving the ambient dimension dependence, this algorithm enjoys a regret (or sample-complexity) bound that is no worse than the online-only minimax optimal bound, with the potential of significant gains if the offline dataset is of high quality \citep{zhou2021nearly, he2023nearly, hu2023nearlynot, agarwal2022voql}. Our result demonstrates the provable benefits of hybrid RL in scenarios where offline samples are much cheaper or much easier to acquire. 
\end{itemize}

To the best of our knowledge, we are the first to show improvements over the aforementioned lower bounds of hybrid RL algorithms (in the same vein as \cite{li2023reward}) in the presence of function approximation, without any explicit requirements on the quality of the behavior policy, and with both the offline-to-online and online-to-offline approaches. Our results are also, at the point of writing, the best bounds available in the literature for hybrid RL in linear MDPs (see Table \ref{tab:best-bounds}).

\begin{table}[t]
\centering
\scalebox{0.85}{
\begin{tabular}{c|c|c}
\toprule  & Upper Bound & Lower Bound 
\tabularnewline
\toprule 
Offline (Error) & $\sqrt{d} \cdot \sum_{h=1}^H \mathbb{E}_{\pi^*}\left\|\phi\left(s_h, a_h\right)\right\|_{\Sigma_{\off,h}^{*-1}}$& $\sqrt{d} \cdot \sum_{h=1}^H \mathbb{E}_{\pi^*}\left\|\phi\left(s_h, a_h\right)\right\|_{\Sigma_{\off,h}^{*-1}}$
\tabularnewline
& $\tcbhighmath[colframe=ballblue]{\leq \sqrt{C^*d^2H^4/N_{\off}}}$ \citep{xiong2023nearly} & $\tcbhighmath[colframe=ballblue]{\geq \sqrt{C^*d^2H^2/N_{\off}}}$ \citep{xiong2023nearly} 
\tabularnewline \hline
Online (Regret) & $\tcbhighmath[colframe=red]{\sqrt{d^2H^3T}}$ \citep{he2023nearly} & $\tcbhighmath[colframe=red]{\sqrt{d^2H^3T}}$ \citep{zhou2021nearly}
\tabularnewline
\toprule
& \multicolumn{2}{c} {Result}
\tabularnewline
\toprule
Hybrid & \multicolumn{2}{c} {$\sqrt{d^2H^7/N}$\citep{wagenmaker2023leveraging}}
\tabularnewline
(Online-to-offline Error) & \multicolumn{2}{c}{$\tcbhighmath[colframe=ballblue]{\sqrt{c_{\off}(\gX_{\off})dH^{3}\min\{c_{\off}(\gX_{\off}),H\}/N_{\off}} + \sqrt{d_{\on}dH^{3}\min\{d_{\on},H\}/N_{\on}}}$ (\textcolor{blue}{Alg. \ref{alg:rappel}})} 
\tabularnewline\hline
Hybrid & \multicolumn{2}{c}{$C^*\sqrt{d^2H^6N_{\on}}$ \citep{song2023hybrid, nakamoto2023calql}}
\tabularnewline
(Offline-to-online Regret) & \multicolumn{2}{c}{$\sqrt{(C^*+c_{\on}(\gX))d^3H^6N_{\on}}$ \citep{amortila2024harnessing}} 
\tabularnewline
& \multicolumn{2}{c}{$\sqrt{c_{\off}(\gX_{\off})dH^5N_{\on}^2/N_{\off}} + \sqrt{d_{\on}dH^5N_{\on}}$ \citep{tan2024natural}} 
\tabularnewline
& \multicolumn{2}{c}{$\tcbhighmath[colframe=red]{\sqrt{c_{\off}(\gX_{\off})^2dH^3N_{\on}^2/N_{\off}} + \sqrt{d_{\on}dH^3N_{\on}}}$ (\textcolor{blue}{Alg. \ref{alg:hyrule}})}
\tabularnewline
\toprule 
\end{tabular}}
\medskip
\caption{Comparisons of our results to the best upper and lower bounds available, and existing results for hybrid RL, in linear MDPs. The inequalities in the offline row hold when the behavior policy satisfies $C^*$-single policy concentrability. Often, offline data is cheaper or easier to obtain. When this happens, $N_{\off} \gg N_{\on}$, and the second term (depending on $N_{\on}=T$) dominates. }
\label{tab:best-bounds}
\vspace{-5mm}
\end{table}

\paragraph{Technical contributions.} 
In this work, we build on recent advancements in offline and online RL with intuitive modifications, demonstrating that it is possible to achieve state-of-the-art sample complexity in a hybrid setting for linear MDPs. At a high level, our improvements in sample complexity are achieved by decomposing the error of interest into offline and online partitions, and optimizing them respectively, following the same idea in \cite{tan2024natural}.
Below, we summarize our specific technical contributions.
\begin{enumerate}
    \item We sharpen the dimensional dependence from $d$ to $d_{\text{on}}$ and $c_{\text{off}}(\mathcal{X}_{\text{off}})$ via projections onto those partitions. The former is accomplished in Algorithm \ref{alg:rappel} by Kiefer-Wolfowitz in Lemma \ref{lem:cov-linear}, and in Algorithm \ref{alg:hyrule} by proving a sharper variant of Lemma B.1 from Zhou and Gu (2022) in Lemma \ref{lem:zhou2022-B1-modified}, using this in Lemma \ref{lem:hybridonline-sumbonuses-xon} to reduce the dimensional dependence in the summation of bonuses, which helps achieve the desired result.
    \item We maintain a $H^3$ dependence for the error or regret for both algorithms, which is non-trivial. This is accomplished in Algorithm \ref{alg:rappel} and for the offline partition in Algorithm \ref{alg:hyrule}  by combining the total variance lemma with a novel truncation argument that rules out “bad” trajectories in Lemma \ref{lem:total-variance-concentration}. 
\end{enumerate}

\section{Preliminaries}
\label{sec:setting}

\vspace{-2mm}
\subsection{Basics of Markov decision processes} 
Consider an episodic MDP denoted by a tuple $\gM = \left(\gS, \gA, H, (\prob_h)_{h=1}^H, (r_h)_{h=1}^H\right)$, where $\gS$ is the state space, $\gA$ the action space, $H$ the horizon, $(\prob_h)_{h=1}^H$ the collection of transition probability kernels where each $\prob_h : \gS \times \gA \to \Delta(\gS)$, and $(r_h)_{h=1}^H$ the collection of reward functions where each $r_h : \gS \times \gA \to [0,1]$. We use $\Delta(\cdot)$ to denote the collection of distributions over a set, and write $[H]=1,...,H$. At each $h\in [H]$, an agent observes the current state $s_h \in \gS$, takes an action $a_h \in \gA$ according to a randomized decision rule $\pi_h : \gS \to \Delta(\gA)$, and observes the reward $r_h$. The next state $s_{h+1}$ then evolves according to $s_{h+1} \sim \prob_h(\cdot \mid s_h, a_h)$. A policy is given by the collection of policies at each horizon $h \in [H]$, $\pi = \{\pi_h\}_{1\leq h\leq H},$ and we write $\Pi$ for the set of all policies. 

We define the value function and Q-functions associated with each policy $\pi \in \Pi$ as follows:
\begin{align}
    \text{for every } (s, h) \in \gS \times [H]: ~~&V_h^\pi(s) := \E_\pi[\textstyle\sum\nolimits_{h'=h}^Hr_{h'}|s_{h}=s],\\
    \text{and for every } (s, a, h) \in \gS \times \gA \times [H]: 
    ~~& Q_h^\pi(s,a) := \E_\pi[\textstyle\sum\nolimits_{h'=h}^Hr _{h'}|s_{h}=s, a_h=a].
\end{align}
$\pi^* = \{\pi^*\}_{h=1}^H$ is the optimal policy attaining the highest value and Q-functions, and we write $V^* = \{V_h^*\}_{h=1}^H$ and $Q^* = \{Q_h^*\}_{h=1}^H$ for the optimal value and Q-functions. 

We consider the setting of hybrid RL, where an agent has access to two sources of data:
\begin{itemize}
    \item $N_{\off}$ independent episodes of length $H$ collected by a behavior policy $\pi_b$
where the $n$-th sample trajectory is a sequence of data $(s^{(n)}_1 , a^{(n)}_1 ,r_1^{(n)}, ..., s^{(n)}_H , a^{(n)}_H ,r^{(n)}_H, s^{(n)}_{H+1});$

\item $N_{\on}$ sequential episodes of online data, where at each episode $n=1,...,N_{\on}$, the algorithm has knowledge of the $N_{\off}$ offline episodes and the previous online episodes $1,...,n-1$. 
\end{itemize}
The quality of the behavior policy  $\pi_b$ is measured by the all-policy and single-policy concentrability coefficients proposed by \cite{xie2023bellmanconsistent, zhan2022offline}:
\begin{defn}[Occupancy Measure]
    For a policy $\pi = \{\pi_h\}_{h=1}^H$, its occupancy measure $d^\pi = \{d^\pi_h\}_{h=1}^H$ corresponds to the collection of distributions over states and actions induced by running $\pi$ within $\gM$, where for some initial distribution $\rho$ and $s_1 \sim \rho$, we have 
    \begin{align}
        d_h^\pi(s,a) := \mathbb{P}(s_h = s, a_h =a \mid s_1 \sim \rho, \pi).
    \end{align}
\end{defn}
\begin{defn}[Concentrability Coefficient]
    For a policy $\pi$, its all-policy and single-policy concentrability coefficients with regard to the occupancy measure of a behavior policy $\pi_b$ are 
    \begin{align}
    C_{\text{all}} :=\sup_\pi \sup_{h,s,a} \frac{d_h^\pi(s,a)}{\mu_h(s,a)}
    ~\text{ and }~ 
    C^*:=\sup_{h,s,a} \frac{d_h^*(s,a)}{\mu_h(s,a)},
    \end{align}
where we write $\mu=\{\mu_h\}_{h=1}^H$ for the occupancy measure of $\pi_b$. 
\end{defn}

\paragraph{Policy learning and regret minimization.} 
The recurring goal of hybrid RL is to either learn an $\epsilon$-optimal policy $\widehat{\pi}$ such that
$V^* - V^{\widehat{\pi}} \leq \epsilon \text{ with high probability},$
or to minimize the regret. 
Here, the regret of an online algorithm, i.e. a map from the history of all previous observations $\gH$ to the set of all policies $\Pi$, $\gL: \gH \to \Pi$ is defined as 
$\text{Reg}_\gL(T) = \E[\sum_{t=1}^{T} (V_1^*(s_1^{(t)}) - \sum_{h=1}^Hr_h^{(t)})].$
Throughout the paper, we shall write $T=N_{\on}$ interchangeably whenever we refer to the number of episodes taken by a regret-minimizing online RL algorithm.

\subsection{Linear MDPs}

Throughout this paper, we consider the setting of linear MDPs first proposed by \cite{yang2019sampleoptimal, jin2019provably}, and further studied in \cite{zanette2021provable, xiong2023nearly, he2023nearly, hu2023nearlynot, wagenmaker2023instancedependent, wagenmaker2023leveraging}. Informally, this is the set of MDPs where the transition probabilities and rewards are linearly parametrizable as functions of known features. 

\begin{aspt}[Linear MDP, \citet{jin2019provably}]
    A tuple $(\mathcal{S}, \mathcal{A}, H, \mathbb{P}, r)$ defines a linear MDP with a (known) feature map $\phi:$ $\mathcal{S} \times \mathcal{A} \rightarrow \mathbb{R}^d$, if for any $h \in[H]$, there exist $d$ unknown signed measures $\mu_h=\left(\mu_h^{(1)}, \cdots, \mu_h^{(d)}\right)$ over $\mathcal{S}$ and an unknown vector $\theta_h \in \mathbb{R}^d$, such that for any $(x, a) \in \mathcal{S} \times \mathcal{A}$, we have $\mathbb{P}_h(\cdot \mid x, a)=$ $\left\langle\phi(x, a), \mu_h(\cdot)\right\rangle, r_h(x, a)=\left\langle\phi(x, a), \theta_h\right\rangle$. Assume $\|\phi(x, a)\| \leq 1$ for all $(x, a) \in \mathcal{S} \times \mathcal{A}$, and $\max \left\{\left\|\mu_h(\mathcal{S})\right\|,\left\|\theta_h\right\|\right\} \leq \sqrt{d}$ for all $h \in[H]$.
    \label{aspt:linear-mdp}
\end{aspt} 

This setting allows for sample-efficient RL due to a collection of reasons. Firstly, linear MDPs are Bellman complete \citep{jin2021bellman}, a common assumption made to ensure sample-efficient RL in the literature \citep{munos2008finite, duan2020minimaxoptimal, fan2020theoretical}.
Secondly, the value and Q-functions are linearly parametrizable in the features, allowing one to learn them via ridge regression.
This allows for sample-efficient, and even minimax-optimal, online \citep{he2023nearly, hu2023nearlynot} and offline \citep{yin2022nearoptimal, xiong2023nearly} reinforcement learning in linear MDPs, despite the requirement of function approximation. However, existing guarantees for hybrid RL in linear MDPs \citep{wagenmaker2023leveraging} are loose \citep{li2023reward}, inspiring the focus of our work.

\paragraph{Further notation.} Write $\phi_{n,h} = \phi(s_h^{(n)}, a_h^{(n)})$ as shorthand for the observed feature vector at episode $n$ and horizon $h$. Let $\boldsymbol{\Lambda}_h = \sum_{n=1}^N \phi_{n,h}\phi_{n,h}^{\top}+\lambda\bf{I}$ and $\boldsymbol{\Lambda}_{\off, h} = \sum_{n=1}^{N_{\off}} \phi_{n,h}\phi_{n,h}^{\top}+\lambda\bf{I}$ be the covariance matrices of the entire dataset and the offline dataset respectively, and $\boldsymbol{\Omega}$ the set of all covariates. We consider two kinds of variance-weighted covariance matrices, namely $\boldsymbol{\Sigma}_{n,h}^{*} = \sum_{n=1}^N \phi_{n,h}\phi_{n,h}^{\top}/\left[\mathbb{V}_h V_{h+1}^*\right]\left(s_h^\tau, a_h^\tau\right)+\lambda\bf{I}$ and $\boldsymbol{\Sigma}_{n,h} = \sum_{n=1}^N \bar{\sigma}_{n,h}^{-2}\phi_{n,h}\phi_{n,h}^{\top}+\lambda\bf{I}$, where $\left[\mathbb{V}_h V_{h+1}^*\right]\left(s_h^\tau, a_h^\tau\right)=\max \left\{1,\left[\operatorname{Var}_h V_{h+1}^*\right](s, a)\right\}$ is the truncated variance of the optimal value function (where $s,a$ are random variables) and $\bar{\sigma}_{n,h}^{-2}$ is the variance estimator from \cite{he2023nearly}.

\subsection{Exploring the state-action space}

The goal of this paper is to develop efficient hybrid RL algorithms for linear MDPs that do not rely on the (full) single-policy concentrability condition, which entails that the behavior policy covers every state-action pair that $\pi^*$ visits.
A natural idea, from \cite{li2023reward, tan2024natural}, is to partition this space into a component that is well-covered by the behavior policy, which we call $\gX_{\off}$, and a component requiring further exploration, which we call $\gX_{\on}$. 
Based on this partition, similarly to  \cite{tan2024natural}, the estimation error or regret of a hybrid RL algorithm can be analyzed on each component  separately. 
We define $\gX_{\on} \cup \gX_{\off} = [H]\times \gS \times \gA$, with their images under the feature map $\Phi_{\off} = \text{Span}(\phi(\gX_{\off,h}))_{h\in [H]} \subseteq \R^d$ and $\Phi_{\on} = \text{Span}(\phi(\gX_{\on,h}))_{h\in [H]}\subseteq \R^d$ being subspaces of dimension $d_{\off}$ and $d_{\on}$ respectively. 
Write $\gP_{\off}, \gP_{\on}$ for the orthogonal projection operators onto these subspaces respectively. 
Let $\lambda_k(M)$ denote the $k$-th largest eigenvalue of a symmetric matrix $M.$
We borrow the definition of the partial offline all-policy concentrability coefficient,
\begin{equation}
    c_{\off}(\gX_{\off})  := \max_h \;{1}\big/{\lambda_{d_{\off}}(\E_{\mu_h}[(\gP_{\off}\phi_h)(\gP_{\off}\phi_h)^{\top}])},
    \label{eq:c_off}
\end{equation}
from \cite{tan2024natural}. This corresponds to the inverse of the $d_{\off}$-th largest eigenvalue of the covariance matrix of the projected feature maps.
Similarly, the partial all-policy analogue of the coverability coefficient from \cite{xie2022role} is
\begin{equation}
    c_{\on}(\gX_{\on}) := \inf_\pi \max_h \; 1\big/{\lambda_{d_{\on}}(\E_{d^\pi_h}[(\gP_{\on}\phi_h)(\gP_{\on}\phi_h)^{\top}])}.
    \label{eq:c_on}
\end{equation}
As we shall see, these quantities characterize the estimation error of our proposed algorithms.  

\section{Algorithms and main results}

We provide two algorithms with {improved} statistical guarantees to tackle the unsolved (Table \ref{tab:best-bounds}) problem of achieving sharp guarantees with hybrid RL in linear MDPs, with different approaches:

\begin{enumerate}
    \item Performing reward-agnostic online exploration \citep{wagenmaker2023leveraging} to augment the offline data, then invoking offline RL \citep{xiong2023nearly} to learn an $\epsilon$-optimal policy on the combined dataset, in the same vein of \cite{li2023reward}. The details can be found in Algorithm \ref{alg:rappel}.

    \item Warm-starting an online RL algorithm \citep{he2023nearly} with parameters estimated from an offline dataset to minimize regret, as in \cite{song2023hybrid}. 
    We include the details in Algorithm \ref{alg:hyrule}.
\end{enumerate}

\subsection{Offline RL after online exploration}


Our algorithm for the online-to-offline approach, Algorithm \ref{alg:rappel}, proceeds as follows. Given access to the offline dataset $\gD_{\off}$, we collect online samples informed by the degree of coverage (or lack thereof) of the offline dataset with a reward-agnostic online exploration algorithm called OPTCOV that was first proposed in \cite{wagenmaker2023instancedependent}. OPTCOV attempts to collect feature vectors such that the minimum eigenvalue of the feature covariance matrix, $\lambda_{\min}(\boldsymbol{\Lambda}_h)$, is bounded below by a tolerance parameter $1/\tau$, and terminates after this is accomplished. We then learn a policy from the combined dataset using a nearly minimax-optimal pessimistic offline RL algorithm from \cite{xiong2023nearly} called LinPEVI-ADV+. 

\begin{algorithm}[t]
    \caption{Reward-Agnostic Pessimistic PAC Exploration-initialized Learning (RAPPEL)}
    \begin{algorithmic}[1]
        \State {\bfseries Input:} Offline dataset $\gD_{\off}$, samples sizes $N_{\on}$, $N_{\off}$, feature maps $\phi_h$, tolerance parameter for reward-agnostic exploration $\tau$.
        \State {\bfseries Initialize:} $\gD_h^{(0)}\leftarrow \emptyset \;\;\forall h \in [H]$, $\lambda = 1/H^2$, $\beta_2 = \tilde{O}(\sqrt{d})$.
        \For{horizon $h=1,...,H$}
            \State Run an exploration algorithm (OPTCOV, \cite{wagenmaker2023instancedependent}) to collect covariates $\boldsymbol{\Lambda}_{h}$ such that $$\max_{\phi_h \in \Phi} \phi_h^{\top}(\boldsymbol{\Lambda}_{h} + \lambda \textbf{I} + \boldsymbol{\Lambda}_{\off,h})^{-1}\phi_h \leq \tau.$$
        \EndFor
        \State {\bfseries Output:} $\widehat{\pi}$ from running a pessimistic offline RL algorithm (LinPEVI-ADV+, \cite{xiong2023nearly}) with hyperparameters $\lambda, \beta_2$ on the combined dataset $\gD_{\off} \cup \{\gD^{(N_{\on})}_h\}_{h \in [H]}$.
    \end{algorithmic}
    \label{alg:rappel}
\end{algorithm}



To employ OPTCOV, one requires a similar assumption to the full-rank covariate assumption from \cite{wagenmaker2023leveraging} that ensures that the MDP is "explorable" enough, but modify it to consider the state-action space splitting framework. The assumption below is only imposed for Algorithm \ref{alg:rappel}.
\begin{aspt}[Full Rank Projected Covariates]
    \label{aspt:full-rank-proj-covariates}
    For any partition $\gX_{\on} \cup \gX_{\off} = [H]\times \gS \times \gA$,
    $$c_{\on}(\gX_{\on}) < \infty, \text{ or equivalently that }
    \inf_\pi \min_h \lambda_{d_{\on}}(\E_{d^\pi_h}[(\gP_{\on}\phi_h)(\gP_{\on}\phi_h)^{\top}]) = \lambda^*_{d_{\on}} > 0.$$
\end{aspt}
Informally, this states that the lowest (best) achievable partial all-policy concentrability coefficient on any online partition must be bounded. 
That is, for any partition, there exists some ``optimal exploration policy'' that ensures that the projected covariates onto the online partition have the same rank as the dimension of the online partition at every timestep.  It essentially requires that there is some policy that collects covariates that span the entire feature space. In practice, this is achievable for any linear MDP via a transformation of the features that amounts to a projection onto the eigenspace corresponding to the nonzero singular values. For example, this is performed for the numerical simulations in Section \ref{sec:sims} -- as in \cite{tan2024natural}, the feature vectors are generated by projecting the $640$-dimensional one-hot state-action encoding onto a $60$-dimensional subspace spanned by the top $60$ eigenvectors of the covariance matrix of the offline dataset. We can then establish the following:
\begin{lem}[Partial Coverability Is Bounded In Linear MDPs]
    \label{lem:cov-linear}
    For any partition $\gX_{\off}, \gX_{\on}$, it satisfies that $c_{\on}(\gX_{\on}) \leq d_{\on}$. Also, there exists at least one partition such that $c_{\off}(\gX_{\off}) = O(d)$.
\end{lem}

The proof of this lemma is deferred to Appendix \ref{app:coverability}. This result  allows us to bound the error on the offline and online partitions by the dimensionality of the partitions, instead of the coverability coefficient.
More specifically, define $\alpha_{\off} := \frac{N_{\off}}{N}$, $\alpha_{\on} := \frac{N_{\on}}{N}$, and the minimal online samples for exploration 
        $$N^*(\tau) := \min _N N \quad\text { s.t. } \inf _{\boldsymbol{\Lambda} \in \boldsymbol{\Omega}} \max _{\boldsymbol{\phi} \in \Phi} \boldsymbol{\phi}^{\top}\left(N(\boldsymbol{\Lambda}+\bar{\lambda} I)+\boldsymbol{\Lambda}_{\mathrm{off}}\right)^{-1} \boldsymbol{\phi} \leq \tau.$$
We now have, with full proof in Appendix \ref{app:hybridoffline} and proof sketch at the end of the subsection, the following:
\begin{thm}[Error Bound for RAPPEL, Algorithm \ref{alg:rappel}]
        \label{thm:hybridoffline}
        For every $\delta \in (0,1)$ and any partition $\gX_{\off}, \gX_{\on}$, when choosing $\tau \leq {\tilde{O}(\max\{d_{\on}/N_{\on}, c_{\off}(\gX_{\off})/N_{\off}\})}$, RAPPEL achieves with probability at least $1-\delta$:
        \begin{align}
        \label{eqn:rappel-res1}
        V_1^*(s)-V_1^{\widehat{\pi}}(s) &\lesssim \sqrt{d}\sum_{h=1}^H \E_{\pi^*}||\phi(s_h, a_h)||_{(\boldsymbol{\Sigma}_{\off,h}^{*}+ \boldsymbol{\Sigma}_{\on,h}^*)^{-1}}\leq \sqrt{d}\sum_{h=1}^H \E_{\pi^*}||\phi(s_h, a_h)||_{\boldsymbol{\Sigma}_{\off,h}^{*-1}},\\
        \label{eqn:rappel-res2}
        V_1^*(s)-V_1^{\widehat{\pi}}(s) &\lesssim \min\left\{\sqrt{\frac{c_{\off}(\gX_{\off})dH^{4}}{N_{\off}}} + \sqrt{\frac{d_{\on}dH^{4}}{N_{\on}}}, \sqrt{\frac{c_{\off}(\gX_{\off})^2dH^3}{N_{\off}\alpha_{\off}}} + \sqrt{\frac{d_{\on}^2dH^3}{N_{\on}\alpha_{\on}}}\right\},
        \end{align}
        given 
        $N \geq \max\left\{{\alpha_{\on}^{4}}{d_{\on}^{-4}},{\alpha_{\off}^{4}}{c_{\off}(\gX_{\off})^{-4}}\right\}\max\{N^*(\tau),\text{poly}(d, H,c_{\off}(\gX_{\off}), \log 1/\delta)\}$. 
\end{thm}

This result, when applied to tabular MDPs with finite states and actions, 
leads to the following:
\begin{cor}\label{cor:rappel-tabular}
    In tabular MDPs, for every $\delta\in (0,1)$, it satisfies that with probability at least $1-\delta$,
    \begin{align}
         V_1^{\star}(s) - V_1^{\widehat{\pi}}(s) \lesssim 
   \sqrt{H^3|\mathcal{S}|^2|\mathcal{A}|}\left(\sqrt{\frac{c_{\off}(\gX_{\off})}{N_{\off}}} + \sqrt{\frac{d_{\on}}{N_{\on}}}\right).
     \end{align} 
\end{cor}

In words, Theorem~\ref{thm:hybridoffline} guarantees that with a burn-in cost polynomial in dimension $d$ and time horizon $H$ and no smaller than $N^*$ (the minimal online samples for any algorithm to achieve our choice of OPTCOV tolerance), we learn an $\epsilon$-optimal policy in at most $$\frac{c_{\off}(\gX_{\off})dH^3\min\{c_{\off}(\gX_{\off}), H\}}{\epsilon^2} + \frac{d_{\on}dH^3\min\{d_{\on}, H\}}{\epsilon^2}$$ trajectories. $N^*$, from \cite{wagenmaker2023leveraging}, is essentially unavoidable in reward-agnostic exploration for linear MDPs.

To compare with prior literature, our result leads to a better worst-case guarantee than the error bound $\sqrt{d^2H^7/N}$ attained in  \cite{wagenmaker2023leveraging} (by at least a factor of $H^{3/2}$), the only other work on hybrid RL in linear MDPs thus far.
While we employ the same online exploration procedure, we combine our exploration phase with an offline learning algorithm LinPEVI-ADV+ from \cite{xiong2023nearly} and conduct a careful analysis.
When comparing with the offline-only and online-only settings, 
Theorem~\ref{thm:hybridoffline} improves upon the offline-only minimax-optimal error bound of $\sqrt{d}\sum_{h=1}^H \E_{\pi^*}||\phi(s_h, a_h)||_{\Sigma_{\off,h}^{*-1}}$ from \cite{xiong2023nearly} as a consequence of  $\Sigma_{\off,h}^{*}+ \Sigma_{\on,h}^* \succeq \Sigma_{\off,h}^{*}$; 
the best offline-only error bound is   $\sqrt{d^2H^4/N_{\off}}$ obtained under the ``well-covered'' assumption (Corollary 4.6, \cite{jin2021pessimism}) that $\lambda_{\min}(\boldsymbol{\Lambda}_{h,\off})\geq\Omega(1/d)$, Theorem~\ref{thm:hybridoffline} enjoys better dimension and horizon dependence
as there is always a partition such that $d_{\on}, c_{\off}(\gX_{\off}) \leq d$ and $d_{\on}H^3\min\{d_{\on}, H\}\leq d^2H^4$. 

We remark that the literature has experienced considerable difficulty in sharpening the horizon dependence to $H^3$ in offline RL for linear MDPs. While \cite{yin2022nearoptimal} and \cite{xiong2023nearly} provide minimax-optimal algorithms for offline RL in linear MDPs, both only manage to achieve a $H^3$ horizon dependence in the special case of tabular MDPs, even under the ``well-covered'' assumption mentioned earlier. We provide the same result in Corollary \ref{cor:rappel-tabular} with proof deferred to Appendix \ref{app:proof-cor-rappel-tab}, but we note that encouragingly, hybrid RL lets us bypass the ``well-covered'' assumption. In Appendix \ref{app:hybridoffline} and \ref{app:misc}, we use a novel truncation argument and the total variance lemma (Lemma C.5 of \cite{jin2018q}) to improve the dependence on $H$, but our result still falls slightly short of a $\sqrt{c_{\off}(\gX_{\off})dH^3/N_{\off}}+\sqrt{d_{\on}dH^3/N_{\on}}$ bound.

\paragraph{Computational efficiency.} In terms of computational efficiency, Algorithm \ref{alg:rappel} inherits the computational costs of the previous proposed algorithms OPTCOV and LinPEVI-ADV+ (\cite{wagenmaker2023instancedependent,xiong2023nearly}. OPTCOV runs in polynomial time $\text{poly}(d,H,c_{\on}(\gX_{\on}), \log 1/\delta)$, and LinPEVI-ADV+ runs in $\tilde{O}(d^3HN|\gA|)$ time when the action space is discrete. Algorithm~\ref{alg:rappel} therefore remains computationally efficient in this case. 

\paragraph{Requirement of choosing $d_{\on}$.} There is the caveat that we require the user to choose the tolerance for OPTCOV. In practice, one can achieve this by performing SVD on the offline dataset and looking at the plot of eigenvalues. One can also choose a tolerance of $O(d/\min\{N_{\off},N_{\on}\})$, but this would not achieve the reduction in the dependence on dimension from $d^2$ to $c_{\off}(\gX_{\off})d, d_{\on}d$. 

\paragraph{Practical benefits of the online-to-offline approach.} Algorithm \ref{alg:rappel} outputs a fixed policy that satisfies a PAC bound. This enables learned policies to be deployed in critical real-world applications, such as in medicine or defense, where randomized policies (as a regret-minimizing online algorithm would provide) are often unacceptable. 

\paragraph{Reward-agnostic hybrid RL.} Secondly, the use of reward-agnostic online exploration in Algorithm \ref{alg:rappel} enables one to use the combined dataset $\gD$ to learn policies for different reward functions offline. As the online exploration is not influenced by any single reward function, the resulting dataset collected satisfies good coverage for any possible reward function even if it is revealed only after exploration, enabling one to use a single dataset to achieve success on many different tasks. This therefore also serves as an algorithm for the related setting of \textit{reward-agnostic hybrid RL}, where the reward function is unknown during online exploration and only revealed to the agent after it.


\paragraph{Proof sketch.}
    The relation~(\ref{eqn:rappel-res1}) in Theorem \ref{thm:hybridoffline} follows from invoking Theorem 2 from \cite{xiong2023nearly} with $N > \Omega(d^2H^6), \lambda=1/H^2, \beta_1 = O(\sqrt{d})$. 
    To establish relation~(\ref{eqn:rappel-res2}), the idea is to first bound the quantity of interest as 
    \begin{align*}
        V_1^*(s)-V_1^{\widehat{\pi}}(s) \leq \sqrt{d} \sum\nolimits_{h=1}^H \max_{\phi_h \in \Phi_{\on}} \sqrt{\phi_h^{\top}\boldsymbol{\Sigma}_h^{*-1}\phi_h} + \sqrt{d} \sum\nolimits_{h=1}^H  \max_{\phi_h \in \Phi_{\off}}\sqrt{ \phi_h^{\top}\boldsymbol{\Sigma}_h^{*-1}\phi_h}.
    \end{align*}
    As $\boldsymbol{\Sigma}_h^{*-1} \preceq H^2 \boldsymbol{\Lambda}_h^{-1}$ (see \cite{xiong2023nearly}), it therefore boils down to controlling $\max_{\phi_h \in \Phi} \phi_h^{\top}\boldsymbol{\Lambda}_{h}^{-1}\phi_h$. 
    Towards this, first, we make the observation that Lemma \ref{lem:cov-linear} suggests that 
    $c_{\on}(\gX_{\on})\leq d_{\on}$. 
    If we run OPTCOV with tolerance $\tilde{O}(\max\{d_{\on}/N_{\on}, c_{\off}(\gX_{\off})/N_{\off}\})$
    on partitions where the above hold, 
    in Lemma \ref{lem:cov-bound},
    we prove that
    $\max_{\phi_h \in \Phi} \phi_h^{\top}\boldsymbol{\Lambda}_{h}^{-1}\phi_h 
    \lesssim \max\left\{c_{\off}(\gX_{\off})/N_{\off}, d_{\on}/N_{\on}\right\}.$ This yields the $c_{\off}(\gX_{\off})dH^4, d_{\on}dH^4$ result. 
    
    To tighten the horizon dependence to $H^3$, we employ an useful truncation argument. More specifically, from the total variance lemma (Lemma C.5 of \cite{jin2018q}), the average variance $\mathbb{V}_h V_{h+1}^*$ is asymptotically on the order of $H$. We therefore define the sets of trajectories $\mathcal{E}_h(\delta_h) = \{\tau\in\mathcal{D}: \left[\mathbb{V}_h V_{h+1}^*\right]\left(s_h^\tau, a_h^\tau\right)\geq H^{1+\delta_h}\}$. The cardinality of each set can be bounded by $|\mathcal{E}_h(\delta_h)|\lesssim {N H^{1-\delta_h}}$, and so truncating at the level where $N H^{1-\delta_h} \approx \min(\frac{N_{\off}}{c_{\off}(\gX_{\off})},\frac{N_{\on}}{d_{\on}})$ leads to $ \min_{\phi_h\in\Phi}\phi_h^{\top}{\Sigma_h^{\star}}\phi_h\gtrsim \frac{1}{N H^2}\min(\frac{N_{\off}}{c_{\off}(\gX_{\off})},\frac{N_{\on}}{d_{\on}})^2$. 
    Putting things together yields the last $c_{\off}(\gX_{\off})^2dH^3, d^2_{\on}dH^3$ result needed, and the theorem then follows. 



\subsection{Online regret minimization}

Thus far, we described an online-to-offline strategy which collects online samples to augment the offline dataset. 
However, in certain critical cases, such as with a doctor treating patients, performance-agnostic online exploration is untenable. One may wish to minimize the regret of the online actions taken while learning a nearly optimal policy. In light of this, we explore another approach inspired by the work of \cite{song2023hybrid, tan2024natural} -- that of warm-starting an online RL algorithm with parameters estimated from an offline dataset. 
We describe this algorithm as \emph{Hybrid Regression for Upper-Confidence Reinforcement Learning (HYRULE)} in Algorithm~\ref{alg:hyrule}. 
We demonstrate that hybrid RL enables provable gains over minimax-optimal online-only regret bounds in the offline-to-online case as well.



In order to warm-start an online RL algorithm with offline dataset, we modify LSVI-UCB++ from \cite{he2023nearly} to take in an offline dataset $\gD_{\off}$ by estimating its parameters from $\gD_{\off}$ with the same formulas it would use as if it had experienced the $N_{\off}$ offline episodes itself. As \cite{tan2024natural} suggest, this can be understood as including the offline episodes in the ``experience replay buffer'' that the algorithm uses to learn parameters. The full version can be found in Appendix \ref{app:hybridonline} as Algorithm \ref{alg:hyrule-full}. 
Doing so allows us prove a regret bound depending on the partial all-policy concentrability coefficient. 

Below we state our theoretical guarantees for this algorithm.  
The proof of this result is deferred to Appendix~\ref{app:hybridonline}, and  
a brief proof sketch is provided at the end of this subsection.
\begin{algorithm}[t]
    \caption{Hybrid Regression for Upper-Confidence Reinforcement Learning (HYRULE)}
    \begin{algorithmic}[1]
        \State {\bfseries Input:} Offline dataset $\gD_{\off}$, samples sizes $N_{\on}$, $N_{\off}$, feature maps $\phi_h$. Regularization parameter $\lambda>0$, confidence radii $\beta, \bar\beta, \tilde\beta$, $t_{\text{last}}=0$.
        \State {\bfseries Initialize:} For $h \in[H]$, estimate $\widehat{\mathbf{w}}_{1, h}, \widecheck{\mathbf{w}}_{1, h}, Q_{1,h}, \widecheck{Q}_{1,h}, \sigma_{1,h}, \bar{\sigma}_{1,h}$ from $\gD_{\off}$, and assign $\boldsymbol{\Sigma}_{0, h}=\boldsymbol{\Sigma}_{1, h}=\boldsymbol{\Sigma}_{\off} + \lambda\mathbf{I} = \sum_{n=1}^{N_{\off}}\bar{\sigma}_{n,h}^{-2}\phi_{n,h}\phi_{n,h}^{\top}+\lambda\mathbf{I}$.
        \For{episodes $t=1,...,T$}
            \State Update optimistic and pessimistic weights $\widehat{\mathbf{w}}_{t, h}, \widecheck{\mathbf{w}}_{t, h}$ for all $h$. 
            \If{there exists a stage $h^{\prime} \in[H]$ such that $\operatorname{det}\left(\boldsymbol{\Sigma}_{t, h^{\prime}}\right) \geq 2 \operatorname{det}\left(\boldsymbol{\Sigma}_{t_{\text {last }}, h^{\prime}}\right)$}
            \State Update optimistic and pessimistic Q-functions $Q_{t, h}(s, a), \widecheck{Q}_{t, h}(s, a)$, set $t_{\text {last }}=t$.
            \EndIf
        \For{horizon $h=1,...,H$}
            \State Play action $a_h^{(t)} \leftarrow \argmax_a Q_{t,h}(s_h^{(t)}, a)$, receive reward $r_h^{(t)}$, next state $s_{h+1}^{(t)}$
            \State Estimate $\sigma_{t, h}$, ${\bar{\sigma}}_{t, h} \leftarrow \max \{\sigma_{t, h}, \sqrt{H}, 2 d^3 H^2||\boldsymbol{\phi}(s_h^{(t)}, a_h^{(t)})||_{\boldsymbol{\Sigma}_{t, h}^{-1}}^{1 / 2}\}$\footnotemark, update $\boldsymbol{\Sigma}_{t+1, h}$.
        \EndFor
        \EndFor
        \State {\bfseries Output:} Greedy policy $\widehat{\pi} = \pi^{Q_{T,h}}$, $\text{Unif}(\pi^{Q_{1,h}},...,\pi^{Q_{T,h}})$ for PAC guarantee. 
    \end{algorithmic}
    \label{alg:hyrule}
\end{algorithm}
\footnotetext{\cite{he2023nearly} write $\bar{\sigma}_{t, h} \leftarrow \max \{\sigma_{t, h}, H,...\}$ instead of $\sqrt{H}$. We believe that this is a typo in their paper, given that in the proof of Lemma B.1, they state right after equation D.7 that $0 \leq \bar{\sigma}_{i h}^{-1} \leq 1 / \sqrt{H}$. Moreover, in the proof of Lemma B.5 the array of equations right after equation D.22, particularly $\left\|\bar{\sigma}_{i, h}^{-1} \phi\left(s_h^i, a_h^i\right)\right\|_2 \leq\left\|\phi\left(s_h^i, a_h^i\right)\right\|_2 / \sqrt{H}$, only holds true if this is $\sqrt{H}$.}


\begin{thm}[Regret Bound for HYRULE, Algorithm \ref{alg:hyrule}]
    \label{thm:hybridonline}
    Given any $\delta \in (0,1)$, for every partition $\gX_{\off}, \gX_{\on}$, if $N_{\on}, N_{\off} = \tilde{\Omega}(d^{13}H^{14})$, the regret of HYRULE is bounded by
    \begin{align*}
        \text{Reg}(N_{\on}) \lesssim  \inf_{\gX_{\off}, \gX_{\on}}\sqrt{c_{\off}(\gX_{\off})^{2}dH^3{N_{\on}^2}/{N_{\off}}} + \sqrt{d_{\on}d H^3N_{\on}},
    \end{align*}
    with probability at least $1-\delta$. 
\end{thm}

\begin{cor}
    By the regret-to-PAC conversion, for any $\delta\in (0,1)$, with probability $1-\delta$, Algorithm \ref{alg:hyrule} achieves a sub-optimality gap of
    \begin{align*}
        V_1^*(s)-V_1^{\widehat{\pi}}(s)\lesssim \inf_{\gX_{\off}, \gX_{\on}}\sqrt{{c_{\off}(\gX_{\off})^2dH^3}/{N_{\off}}} + \sqrt{{d_{\on}d H^3}/{N_{\on}}}.
    \end{align*}
\end{cor}

To understand this result, we first note that bounding the regret over all possible partitions yields an improvement over the $\sqrt{d^2H^3N_{\on}}$ regret bound originally obtained by \cite{he2023nearly}, as we can simply take $\gX_{\on}=\gX_{\on}'=\gX$ to recover this result. 
In the scenario where offline samples are largely available (where $N_{\off} \gg N_{\on}$), it is possible to achieve significant improvements over online-only learning. 
Furthermore, in view of Lemma \ref{lem:cov-linear}, there always exists a partition such that $c_{\off}(\gX_{\off}), d_{\on} \leq d$. This result therefore yields provable improvements over the minimax-optimal online regret bound in linear MDPs \citep{zhou2021nearly, he2023nearly, hu2023nearlynot, agarwal2022voql}.

Additionally, Theorem \ref{thm:hybridonline} shows that Algorithm \ref{alg:hyrule} attains the best known regret bound in hybrid RL for linear MDPs, as we illustrate in Table \ref{tab:best-bounds}. The current best known result is that of \cite{tan2024natural}, with a dependence of $\sqrt{c_{\mathrm{off}}\left(\mathcal{X}_{\mathrm{off}}\right) d H^5 N_{\mathrm{on}}^2 / N_{\mathrm{off}}}+\sqrt{d_{\mathrm{on}} d H^5 N_{\mathrm{on}}}$. Notably, we achieve the same a reduction in the dimension dependence on the online partition from $d^2$ to $d_{\on}d$ that \cite{tan2024natural} do by proving a sharper variant of Lemma B.1 from Zhou and Gu (2022) in Lemma \ref{lem:zhou2022-B1-modified}, using this in Lemma \ref{lem:hybridonline-sumbonuses-xon} to reduce the dimensional dependence in the summation of bonuses. \cite{song2023hybrid} and \cite{amortila2024harnessing}, on the other hand, have bounds on the order of $C^*\sqrt{d^2H^6N_{\on}}$ and $\sqrt{(C^*+c_{\on}(\gX))d^3H^6N_{\on}}$ respectively. We produce a better bound than \cite{tan2024natural, song2023hybrid, amortila2024harnessing} by at least a factor of $H^2$ by combining the total variance lemma and a novel truncation argument that rules out “bad” trajectories in Lemma \ref{lem:total-variance-concentration}, which allows us to maintain a desirable $H^3$ dependence on both partitions. 

\paragraph{Computational efficiency.} In terms of computational efficiency, when the action space is finite and of cardinality $|\mathcal{A}|$ the computational complexity of Algorithm \ref{alg:hyrule} is of order $\widetilde{O}\left(d^4 H^3 N |\mathcal{A}| \right)$, as outlined in \cite{he2023nearly}. Algorithm \ref{alg:hyrule} is therefore computationally efficient and runs in polynomial time in this case. When the action space is continuous, one may need to solve an optimization problem over the continuous action space, making the computational complexity highly problem-dependent.

\paragraph{Algorithm \ref{alg:hyrule} is unaware of the partition.} 
Unlike Algorithm \ref{alg:rappel}, Algorithm \ref{alg:hyrule} is fully unaware of the choice of partition, and there is therefore no need to estimate $d_{\on}$ or any relevant analogue to the choice of tolerance for OPTCOV. The regret bound therefore automatically adapts to the best possible partition, even though Algorithm \ref{alg:hyrule} is unaware of it.

\paragraph{Practical benefits of the offline-to-online approach.} While Algorithm \ref{alg:hyrule} only satisfies a PAC bound with a randomized policy, it minimizes the regret of the actions it takes. This enables the algorithm to be deployed in situations where its performance during online exploration is of critical importance, e.g. in applications like mobile health \citep{NahumShani2017JustinTimeAI}.


\paragraph{Technical challenges.} Although Algorithm \ref{alg:hyrule} is a straightforward generalization of LSVI-UCB++ in \cite{he2023nearly}, with $\Sigma_0$ initialized with the offline dataset, we had to decompose the regret into the regret on the offline and online partitions to achieve the regret guarantee in Theorem 2. In the process, we faced the following challenges:
\begin{itemize}
    \item Bounding the regret on the offline partition was challenging, as we were not able to utilize the technique that was used in He et. al (2023). Instead, we bounded the regret with the maximum eigenvalue of $\Sigma_{off,h}^{-1}$. To maintain a $H^3$ dependence on the offline partition, we had to use a truncation argument in Lemma \ref{lem:total-variance-concentration} that we also deployed in proving the regret guarantee of Algorithm \ref{alg:rappel}.
    \item Bounding the regret on the online partition allowed us to use an analysis that was close to that of \cite{he2023nearly}. However, directly following the argument of \cite{he2023nearly} would have left us with a $d^2H^3$ dependence in Theorem \ref{thm:hybridonline}. To reduce the dimensional dependence to $d_{\on}d$, we prove a sharper variant of Lemma B.1 from Zhou and Gu (2022) in Lemma \ref{lem:zhou2022-B1-modified}, using this in Lemma \ref{lem:hybridonline-sumbonuses-xon} to reduce the dimensional dependence in the summation of bonuses enough to achieve the desired result. Without the above two techniques, one could have used a simpler analysis to achieve a far looser $\sqrt{c_{\off}(\gX_{\off})^2d^6H^8 N_{\on}^2/N_{\off}} + \sqrt{d^2H^3}$ regret bound by using the maximum magnitude of the variance weights for the offline partition and the analysis from \cite{he2023nearly} verbatim for the online partition, but this would not have yielded the same improvement.
\end{itemize}
We accordingly provide a proof sketch below.

\paragraph{Proof sketch.} We first adopt the regret decomposition as in \cite{he2023nearly} and bound  
    $$\text{Reg}(T)\lesssim \sqrt{H^3 T} + \textstyle{\sum}_{h,t }\beta\|\boldsymbol{\Sigma}_{t,h}^{-1/2}\phi_h(s_h^{(t)},a_h^{(t)})\mathbbm{1}_{\gX_{\off}}\|_2 + \textstyle{\sum}_{h,t }\beta\|\boldsymbol{\Sigma}_{t,h}^{-1/2}\phi_h(s_h^{(t)},a_h^{(t)})\mathbbm{1}_{\gX_{\on}}\|_2.$$
    It then boils down to controlling the second and the third term separately. 
    We prove in Lemma~\ref{lem:hybridoffline-sumbonuses} that the sum of bonuses on the offline partition can be bounded by $\sum_h \sqrt{dN_{\on}\frac{N_{\on}}{N_{\off}}\max_{\phi_h \in \Phi_{\off}} \phi_h^{\top}\bar{\bf{\Sigma}}_{\off,h}^{-1}\phi_h}.$ 
    To further control this term, we then show in Lemma \ref{lem:hybridoffline-concentrability} that, for any partition $\gX_{\off}, \gX_{\on}$, $\sum_h\max_{\phi_h \in \Phi_{\off}} \sqrt{\phi_h^{\top}\bar{\bf{\Sigma}}_{\off,h}^{-1}\phi_h} \lesssim {c_{\off}(\gX_{\off})^2H^3}$. Putting things together, the second term can be controlled as 
    $$\beta\textstyle{\sum}_{h,t }\|\boldsymbol{\Sigma}_{t,h}^{-1/2}\phi_h(s_h^{(t)},a_h^{(t)})\mathbbm{1}_{\gX_{\off}}\|_2 \lesssim \sqrt{c_{\off}(\gX_{\off})^2dH^3N_{\on}^2/{N_{\off}}}.$$
    
    With respect to the third term, Lemma \ref{lem:hybridonline-sumbonuses-xon} (a sharpened version of Lemma E.1 in \cite{he2023nearly}), combined with the Cauchy-Schwartz inequality, yields
    $$\beta\textstyle{\sum}_{h,t }\|\boldsymbol{\Sigma}_{t,h}^{-1/2}\phi_h(s_h^{(t)},a_h^{(t)})\mathbbm{1}_{\gX_{\on}}\|_2 \lesssim d^4 H^8+\beta d^7 H^5+\beta \sqrt{d_{\on} H T+d_{\on} H \sum\nolimits_{h,t} \sigma_{t, h}^2}.$$
    Lastly, the total variance lemma (Appendix B, \cite{he2023nearly}) further suggests $\sum_{h,t} \sigma_{t, h}^2 \leq \widetilde{O}\left(H^2 T+d^{10.5} H^{16}\right)$. 
    Taking everything collectively establishes the desired result.

\section{Numerical experiments}
\label{sec:sims}

To demonstrate the benefits of hybrid RL in the offline-to-online and online-to-offline settings, we implement Algorithms \ref{alg:rappel} and \ref{alg:hyrule} on the scaled-down Tetris environment (as in \cite{tan2024natural}). This is a $6$-piece wide Tetris board with pieces no larger than $2 \times 2$, where the action space consists of four actions, differentiated by the degree of rotation in 90 degree intervals and  
the reward is given by penalizing any increases in the height of the stack from a tolerance of $2$ blocks. 
The offline dataset consists of $200$ trajectories generated from a uniform behavior policy. 
As in \cite{tan2024natural}, the feature vectors are  generated by projecting the $640$-dimensional one-hot state-action encoding onto a $60$-dimensional subspace spanned by the top $60$ eigenvectors of the covariance matrix of the offline dataset.\footnote{For simplicity in implementation, we implement LSVI-UCB++ \citep{he2023nearly} for Algorithm \ref{alg:hyrule} as-is, while substituting LSVI-UCB \citep{jin2019provably} for FORCE \citep{wagenmaker2022firstorder} within OPTCOV and LinPEVI-ADV for LinPEVI-ADV+ \citep{xiong2023nearly}.}

\begin{figure}[H]
    \centering
    \includegraphics[width=\textwidth]{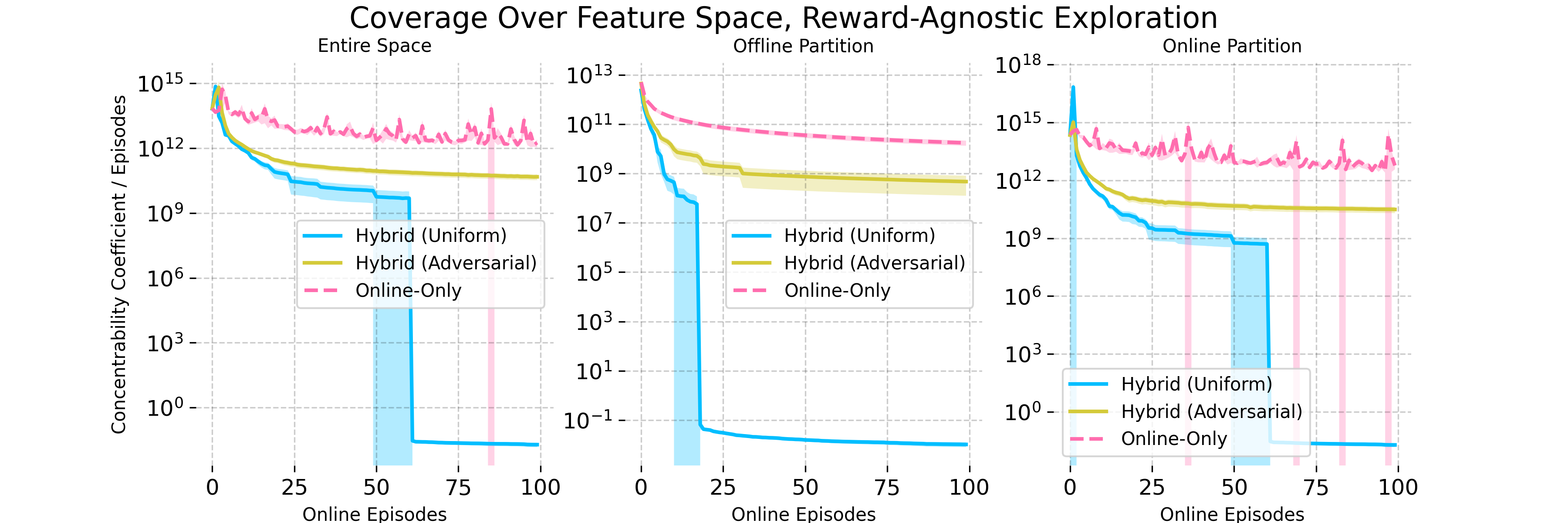}
    \caption{Coverage achieved by OPTCOV with 200 trajectories of offline data collected under a uniform and an adversarial behavior policy, and with no offline data. Results averaged over $30$ trials, with the shaded area depicting $1.96$-standard errors. Lower is better.}
    \label{fig:cov-agnostic}
\end{figure}

Figure \ref{fig:cov-agnostic} depicts the coverage (defined by $1/\lambda_{\min}(\mathbf{\Lambda}), 1/\lambda_{d_{\off}}(\mathbf{\Lambda}_{\off}), 1/\lambda_{d_{\on}}(\mathbf{\Lambda}_{\on})$) achieved by the reward-agnostic exploration algorithm, OPTCOV, when initialized respectively with 200 trajectories from (1) a uniform behavioral policy, (2) an adversarial behavior policy obtained by the negative of the weights of a fully-trained agent under Algorithm \ref{alg:rappel}, and (3) no offline trajectories at all for fully online learning. It shows that although hybrid RL with the uniform behavior policy achieves the best coverage throughout as expected, even hybrid RL with adversarially collected offline data achieves better coverage than online-only exploration. This demonstrates the potential of hybrid RL as a tool for taking advantage of poor quality offline data. 

\begin{figure}[H]
    \centering
    \includegraphics[width=\textwidth/\real{1.9}]{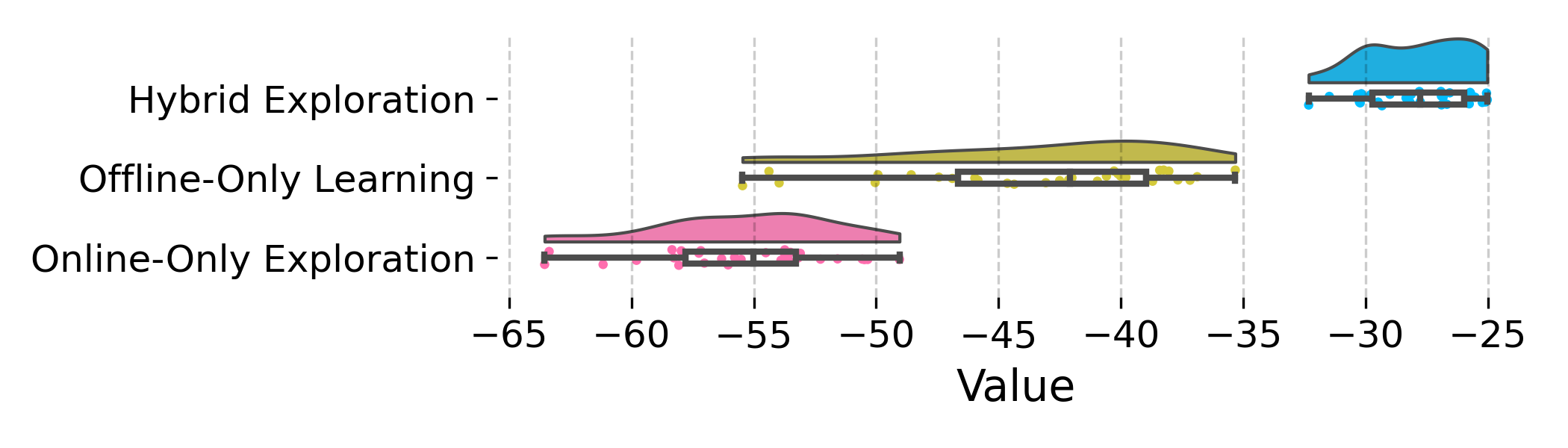}
    \caption{Value of policies learned by applying LinPEVI-ADV to the hybrid, offline, and online datasets, with an adversarial behavior policy. The reward is negative as it is the negative of the excess height. Results over $30$ trials. Higher is better.}
    \label{fig:reward-agnostic}
\end{figure}

In Figure \ref{fig:reward-agnostic}, one can observe that hybrid RL demonstrates strong benefits in the online-to-offline setting when the behavior policy is of poor quality. When applying LinPEVI-ADV to the hybrid dataset of $200$ trajectories and $100$ online trajectories, $300$ trajectories of adversarially collected offline data, and $300$ trajectories of online data under reward-agnostic exploration, we see that the hybrid dataset is most conducive for learning. Additionally, without a warm-start from offline data, online-only reward-agnostic exploration performs worse than the adversarially collected offline data due to significant burn-in costs. Hybrid RL therefore, in this instance, performs better than both offline-only and online-only learning alone.

\begin{figure}[H]
    \centering
    \includegraphics[width=0.8\textwidth]{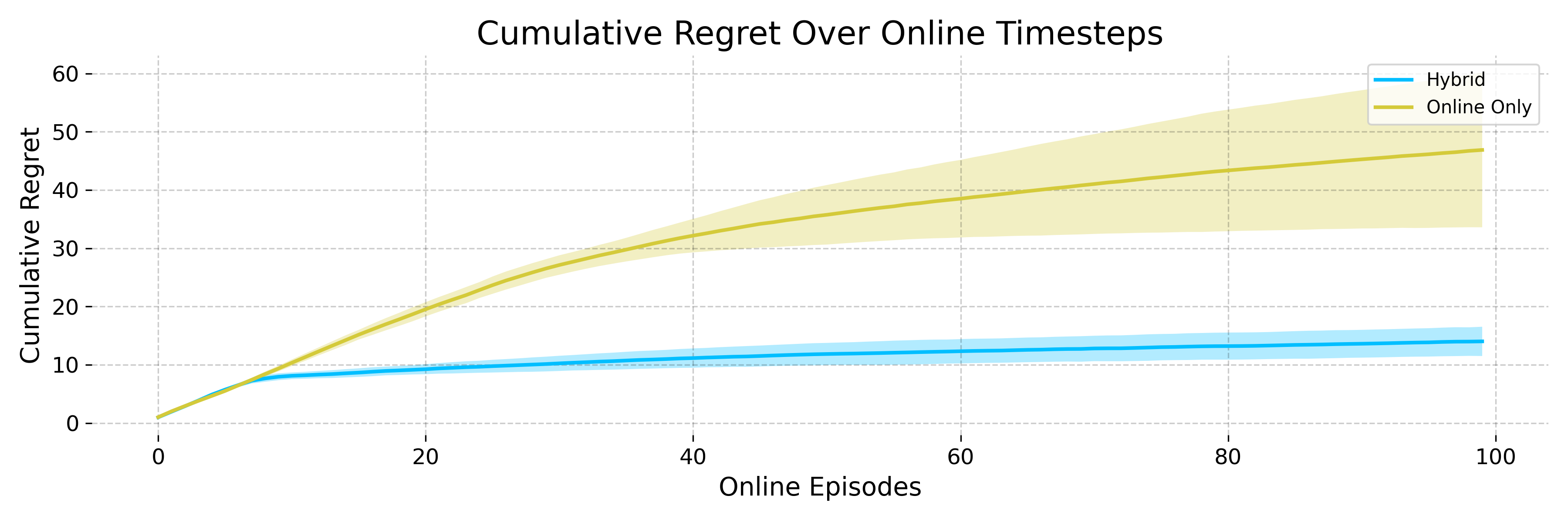}
    \includegraphics[width=0.8\textwidth]{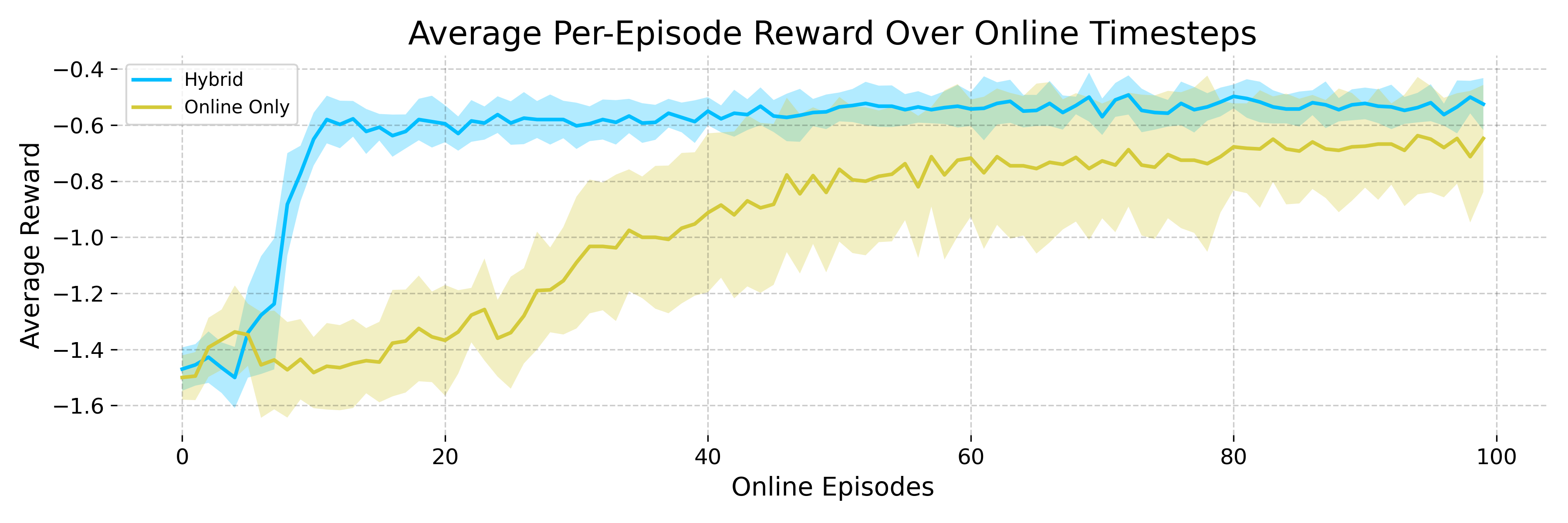}
    \caption{Comparison of LSVI-UCB++ and Algorithm \ref{alg:hyrule}. Results averaged over 10 trials, with $1$-standard deviation error bars over 10 trials. }
    \label{fig:regret}
\end{figure}

In Figure \ref{fig:regret}, we compare the performances of LSVI-UCB++ and Algorithm \ref{alg:hyrule}. 
It can be seen from the figure that initializing a regret-minimizing online algorithm (LSVI-UCB++, \citep{he2023nearly}) with an offline dataset as in Algorithm \ref{alg:hyrule} yields lower regret than the same algorithm without an offline dataset. This shows that even a nearly minimax-optimal online learning algorithm can stand to benefit from being initialized with offline data.

\section{Discussion, limitations and future work}
\label{sec:conclusions-limitations}

In this paper, we develop two hybrid RL algorithms for linear MDPs with desirable statistical guarantees. The first performs reward-agnostic online exploration to fill in gaps in the offline dataset before using offline RL to learn an $\epsilon$-optimal policy from the combined dataset, while the second warm-starts online RL with parameters estimated from an offline dataset. Both algorithms demonstrate provable gains over the minimax-optimal rates in offline or online-only reinforcement learning, and provide the sharpest worst-case bounds for hybrid RL in linear MDPs thus far. 

Throughout this paper, we have used both optimism and pessimism in our algorithm design. Other work in hybrid RL \citep{song2023hybrid, nakamoto2023calql, li2023reward, tan2024natural, amortila2024harnessing, wagenmaker2023leveraging} uses optimism, pessimism, or sometimes even neither. We conjecture that optimism is still helpful in aiding online exploration within hybrid RL and that pessimism helps in hybrid RL when learning from a combined dataset. However, determining if or when optimism or pessimism is beneficial in hybrid RL remains an open question.

Achieving a $H^3$ horizon dependence in offline RL for linear MDPs has proven challenging. Even under strong coverage assumptions, \cite{yin2022nearoptimal} and \cite{xiong2023nearly} only manage to achieve a $H^3$ horizon dependence for tabular MDPs. Obtaining a $\sqrt{d^2H^3/N}$ bound is an open problem. 

Furthermore, while Algorithm \ref{alg:rappel} improves upon the offline-only error bound in \cite{xiong2023nearly} and Algorithm \ref{alg:hyrule} improves upon the online-only regret bound in \cite{he2023nearly, zhou2021nearly}, we still desire a single algorithm that improves upon both the best possible offline-only and online-only rates at once. Additionally, the burn-in costs for Algorithms \ref{alg:rappel} and \ref{alg:hyrule} are nontrivial. The former is inherited from the OPTCOV algorithm of \cite{wagenmaker2023instancedependent}, while the latter is inherited from \cite{he2023nearly} and the truncation argument. Improving the former by devising new reward-agnostic exploration algorithms for linear MDPs, perhaps in the vein of \cite{li2023minimaxoptimal}, would be welcome. 

While we tackle the setting of linear MDPs, it remains a first step towards showing that hybrid RL breaks minimax-optimal barriers in the presence of function approximation. Further work in this vein on other types of function approximation would be an interesting contribution to the literature. 

\section*{Acknowledgements}

Y.~Wei is supported in part by the NSF grants CAREER award DMS-2143215, CCF-2106778, CCF-2418156, and the Google Research Scholar Award.

\bibliographystyle{apalike}
\bibliography{ref}

\appendix

 \newpage

\section{Unabridged versions of our algorithms}

\begin{algorithm}[H]
    \caption{Reward-Agnostic Exploration-initialized Pessimistic PAC Learning (RAPPEL, Full)}
    \begin{algorithmic}[1]\State {\bfseries Input:} Offline dataset $\gD_{\off}$, samples sizes $N_{\on}$, $N_{\off}$, feature maps $\phi_h$, , tolerance parameter for reward-agnostic exploration $\tau$.
        \State {\bfseries Initialize:} $\gD_h^{(0)}\leftarrow \emptyset \;\;\forall h \in [H]$, $\lambda = 1/H^2$, $\beta_2 = \tilde{O}(\sqrt{d})$. Set functions to optimize $
f_i(\boldsymbol{\Lambda})=\eta_i^{-1} \log\left(\textstyle\sum_{\phi \in \Phi} \exp({\eta_i\|\boldsymbol{\phi}\|_{\mathbf{A}_i(\boldsymbol{\Lambda})^{-1}}^2})\right), \mathbf{A}_i(\boldsymbol{\Lambda})=\boldsymbol{\Lambda}+(T_i K_i)^{-1}( \boldsymbol{\Lambda}_{0, i}+ \boldsymbol{\Lambda}_{\off})
$
for some $\boldsymbol{\Lambda}_{0, i}$ satisfying $\boldsymbol{\Lambda}_{0, i} \succeq \boldsymbol{\Lambda}_0$ for all $i$, and $\eta_i=2^{2 i / 5}$.

        \textcolor{blue}{\textbf{Exploration Phase:} Run an exploration algorithm (OPTCOV, \cite{wagenmaker2023instancedependent}) to collect covariates $\boldsymbol{\Lambda}_{h}$ such that $\max_{\phi_h \in \Phi} \phi_h^{\top}(\boldsymbol{\Lambda}_{h} + \lambda \textbf{I} + \boldsymbol{\Lambda}_{\off,h})^{-1}\phi_h \leq \tau.$}
        \For{$i=1,2,3,...$}
            \State Set the number of iterates $T_i \leftarrow 2^i$, episodes per iterate $K_i \leftarrow 2^i$.
            \State Play any policy for $K_i$ episodes to collect covariates $\boldsymbol{\Gamma}_0$ and data $\mathfrak{D}_0$.
            \State Initialize covariance matrix $\boldsymbol{\Lambda}_1 \gets \boldsymbol{\Gamma}_0/K$.
            \For{$t=1,...,T_i$}
                \If{$\sum_{j=1}^i T_j K_j \geq N_{\on}$}
                    \State \textbf{break}
                \EndIf
                \State Run FORCE \citep{wagenmaker2022firstorder} or another regret-minimizing algorithm on the exploration-focused synthetic reward $g_h^{(t)}(s,a) \propto \text{tr}(-\nabla_{\boldsymbol{\Lambda}} f_i(\boldsymbol{\Lambda})|_{\boldsymbol{\Lambda}=\boldsymbol{\Lambda}_t \phi(s,a)\phi(s,a)^{\top}})$.
                \State Collect covariates $\boldsymbol{\Gamma}_t$, data $\mathfrak{D}_t$. 
                \State Perform Frank-Wolfe update: $\boldsymbol{\Gamma}_{t+1} \gets (1-\frac{1}{t+1})\boldsymbol{\Lambda}_{t} + \frac{1}{t+1}\boldsymbol{\Gamma}_{t}/{K_i}$.
            \EndFor
            \State Assign $\widehat{\boldsymbol{\Lambda}_{i,h}} \gets \boldsymbol{\Lambda}_{T_i+1}, \mathfrak{D}_i \gets \cup_{t=0}^{T_i}\mathfrak{D}_t$.
            \State Set $\boldsymbol{\Lambda}_h = \widehat{\boldsymbol{\Lambda}_{i,h}}, \gD_{\on} = \mathfrak{D}_i$.
            \If{$f_i(\widehat{\boldsymbol{\Lambda}_i})\leq K_iT_i\tau$}
                \State \textbf{break}
            \EndIf
        \EndFor

    \textcolor{blue}{\textbf{Planning Phase:} Estimate $\widehat{\pi}$ using a pessimistic offline RL algorithm (LinPEVI-ADV+, \cite{xiong2023nearly}) with hyperparameters $\lambda, \beta_2$ on the combined dataset $\gD_{\off} \cup \{\gD^{(N_{\on})}_h\}_{h \in [H]}$.}
    \State Split the dataset $\gD_{\off} \cup \{\gD^{(N_{\on})}_h\}_{h \in [H]}$into $\gD$ and $\gD'$. Estimate, on $\gD'$, $$
        \begin{aligned}
        & \widetilde{\beta}_{h, 2}=\underset{\beta \in \mathbb{R}^d}{\operatorname{argmin}} \sum_{\tau \in \mathcal{D}^{\prime}}\Big[\left\langle\phi\left(s_h^\tau, a_h^\tau\right), \beta\right\rangle-\big(\widehat{V}_{h+1}^{\prime}\big)^2\left(s_{h+1}^\tau\right)\Big]^2+\lambda\|\beta\|_2^2, \\
        & \widetilde{\beta}_{h, 1}=\underset{\beta \in \mathbb{R}^d}{\operatorname{argmin}} \sum_{\tau \in \mathcal{D}^{\prime}}\left[\left\langle\phi\left(s_h^\tau, a_h^\tau\right), \beta\right\rangle-\widehat{V}_{h+1}^{\prime}\left(s_{h+1}^\tau\right)\right]^2+\lambda\|\beta\|_2^2.\\
        &\widehat{\sigma}_h^2(s, a):=\max \Big\{1,\big[\phi(s, a)^{\top} \widetilde{\beta}_{h, 2}\big]_{\left[0, H^2\right]}-\big[\phi(s, a)^{\top} \widetilde{\beta}_{h, 1}\big]_{[0, H]}^2-\tilde{O}\Big(\frac{d H^3}{\sqrt{N \kappa}}\Big)\Big\}.
        \end{aligned}
        $$
    \For{$h=1,...,H$}
        \State Compute covariance matrix $\boldsymbol{\Sigma}_h=\sum_{\tau \in \mathcal{D}} \phi\left(s_h^\tau, a_h^\tau\right) \phi\left(s_h^\tau, a_h^\tau\right)^{\top} / \widehat{\sigma}_h^2\left(s_h^\tau, a_h^\tau\right)+\lambda \bf{I}_d.$
        \State Compute weights $\widehat{w}_h=\boldsymbol{\Sigma}_h^{-1}\Big(\sum_{\tau \in \mathcal{D}} \phi\left(s_h^\tau, a_h^\tau\right) \frac{r_h^\tau+\widehat{V}_{h+1}\left(s_{h+1}^\tau\right)}{\widehat{\sigma}_h^2\left(s_h^\tau, a_h^\tau\right)}\Big).$
        \State Compute pessimistic penalty $\Gamma_h(\cdot, \cdot) \leftarrow \beta_2\|\phi(\cdot, \cdot)\|_{\boldsymbol{\Sigma}_h^{-1}}.$
       \State Compute pessimistic Q-function $\widehat{Q}_h(\cdot, \cdot) \leftarrow\big\{\phi(\cdot, \cdot))^{\top} \widehat{w}_h-\Gamma_h(\cdot, \cdot)\big\}_{[0, H-h+1]}.$
        \State Set $\widehat{\pi}_h(\cdot \mid \cdot) \leftarrow \arg \max _{\pi_h}\big\langle\widehat{Q}_h(\cdot, \cdot), \pi_h(\cdot \mid \cdot)\big\rangle_{\mathcal{A}}$, $\widehat{V}_h(\cdot) \leftarrow\big\langle\widehat{Q}_h(\cdot, \cdot), \widehat{\pi}_h(\cdot \mid \cdot)\big\rangle_{\mathcal{A}}.$
    \EndFor
    \State {\bfseries Output:} $\widehat{\pi}$.
    \end{algorithmic}
    \label{alg:rappel-full}
\end{algorithm}
\begin{algorithm}[H]
    \caption{Hybrid Regression for Upper-Confidence Reinforcement Learning (HYRULE, Full)}
    \begin{algorithmic}[1]
        \State {\bfseries Input:} Offline dataset $\gD_{\off}$, samples sizes $N_{\on}$, $N_{\off}$, feature maps $\phi_h$. Regularization parameter $\lambda>0$, confidence radii $\beta, \bar\beta, \tilde\beta$, $t_{\text{last}}=0$.
        \State {\bfseries Initialize:} For $h \in[H]$, estimate $\widehat{\mathbf{w}}_{1, h}, \widecheck{\mathbf{w}}_{1, h}, Q_{1,h}, \widecheck{Q}_{1,h}, \sigma_{1,h}, \bar{\sigma}_{1,h}$ from $\gD_{\off}$ with the same formulas outlined below, and assign $\boldsymbol{\Sigma}_{0, h}=\boldsymbol{\Sigma}_{1, h}=\boldsymbol{\Sigma}_{\off} + \lambda\mathbf{I} = \sum_{n=1}^{N_{\off}}\bar{\sigma}_{n,h}^{-2}\phi_{n,h}\phi_{n,h}^{\top}+\lambda\mathbf{I}$.
        \For{episodes $t=1,...,T$}
        \State Receive the initial state $s_{1}^{(t)}$.
        \For{horizon $h=1,...,H$}
            \State $\widehat{\mathbf{w}}_{k, h}=\mathbf{\Sigma}_{t, h}^{-1} \sum_{i=1}^{t-1} \bar{\sigma}_{i, h}^{-2} \boldsymbol{\phi}(s_h^{(i)}, a_h^{(i)}) V_{t, h+1}(s_{h+1}^{(i)})$.
            \State $\widecheck{\mathbf{w}}_{t, h}=\boldsymbol{\Sigma}_{t, h}^{-1} \sum_{i=1}^{t-1} \bar{\sigma}_{i, h}^{-2} \boldsymbol{\phi}(s_h^{(i)}, a_h^{(i)}) \widecheck{V}_{t, h+1}(s_{h+1}^{(i)})$.
            \If{there exists a stage $h^{\prime} \in[H]$ such that $\operatorname{det}\left(\boldsymbol{\Sigma}_{t, h^{\prime}}\right) \geq 2 \operatorname{det}\left(\boldsymbol{\Sigma}_{t_{\text {last }}, h^{\prime}}\right)$}
            \State $Q_{t, h}(s, a)=\min \left\{r_h(s, a)+\widehat{\mathbf{w}}_{t, h}^{\top} \boldsymbol{\phi}(s, a)+\beta \sqrt{\boldsymbol{\phi}(s, a)^{\top} \boldsymbol{\Sigma}_{t, h}^{-1} \boldsymbol{\phi}(s, a)}, Q_{t-1, h}(s, a), H\right\}$.
            \State $\widecheck{Q}_{t, h}(s, a)=\max \left\{r_h(s, a)+\widecheck{\mathbf{w}}_{t, h}^{\top} \boldsymbol{\phi}(s, a)-\bar{\beta} \sqrt{\boldsymbol{\phi}(s, a)^{\top} \boldsymbol{\Sigma}_{t, h}^{-1} \boldsymbol{\phi}(s, a)}, \widecheck{Q}_{t-1, h}(s, a), 0\right\} $.
            \State  Set the last updating episode $t_{\text {last }}=t$.
            \Else 
            \State $ Q_{t, h}(s, a)=Q_{t-1, h}(s, a)$, $\widecheck{Q}_{t, h}(s, a)=\widecheck{Q}_{t-1, h}(s, a)$.
            \EndIf
            \State $V_{t, h}(s)=\max _a Q_{t, h}(s, a)$, $\widecheck V_{t, h}(s)=\max _a \widecheck Q_{t, h}(s, a)$.
        \EndFor
        \For{horizon $h=1,...,H$}
            \State Play action $a_h^{(t)} \leftarrow \argmax_a Q_{t,h}(s_h^{(t)}, a)$.
            \State Estimate $\sigma_{t, h}=\sqrt{\left[\overline{\mathbb{V}}_{t, h} V_{t, h+1}\right]\left(s_h^{(t)}, a_h^{(t)}\right)+E_{t, h}+D_{t, h}+H},$ setting  $E_{t,h}$ and $D_{t,h}$:
            \begin{align*}
E_{t, h}= & \min \left\{\widetilde{\beta}\left\|\boldsymbol{\Sigma}_{t, h}^{-1 / 2} \boldsymbol{\phi}(s_h^{(t)}, a_h^{(t)})\right\|_2, H^2\right\}+\min \left\{2 H \bar{\beta}\left\|\boldsymbol{\Sigma}_{t, h}^{-1 / 2} \boldsymbol{\phi}(s_h^{(t)}, a_h^{(t)})\right\|_2, H^2\right\}, \\
D_{t, h}= & \min \Bigg\{4 d ^ { 3 } H ^ { 2 } \Bigg(\widehat{\mathbf{w}}_{t, h}^{\top} \boldsymbol{\phi}(s_h^{(t)}, a_h^{(t)})-\widecheck{\mathbf{w}}_{t, h}^{\top} \boldsymbol{\phi}(s_h^{(t)}, a_h^{(t)}) \\
&\qquad + 2 \bar{\beta} \sqrt{\boldsymbol{\phi}(s_h^{(t)}, a_h^{(t)})^{\top} \boldsymbol{\Sigma}_{t, h}^{-1} \boldsymbol{\phi}(s_h^{(t)}, a_h^{(t)})}\Bigg), d^3 H^3\Bigg\}.
\end{align*}
            \State ${\bar{\sigma}}_{t, h} \leftarrow \max \left\{\sigma_{t, h}, \sqrt{H}, 2 d^3 H^2\left\|\boldsymbol{\phi}\left(s_h^{(t)}, a_h^{(t)}\right)\right\|_{\boldsymbol{\Sigma}_{t, h}^{-1}}^{1 / 2}\right\}$\footnotemark.
           \State $\boldsymbol{\Sigma}_{t+1, h}=\boldsymbol{\Sigma}_{t, h}+\bar{\sigma}_{t, h}^{-2} \boldsymbol{\phi}\left(s_h^{(t)}, a_h^{(t)}\right) \boldsymbol{\phi}\left(s_h^{(t)},a_h^{(t)}\right)^{\top}$.
            \State Receive reward $r_h^{(t)}$, next state $s_{h+1}^{(t)}$.
        \EndFor
        \EndFor
        \State {\bfseries Output:} Greedy policy $\widehat{\pi} = \pi^{Q_{T,h}}$, $\text{Unif}(\pi^{Q_{1,h}},...,\pi^{Q_{T,h}})$ for PAC guarantee. 
    \end{algorithmic}
    \label{alg:hyrule-full}
\end{algorithm}
\footnotetext{\cite{he2023nearly} write $\bar{\sigma}_{t, h} \leftarrow \max \{\sigma_{t, h}, H,...\}$ instead of $\sqrt{H}$. We believe that this is a typo in their paper, given that in the proof of Lemma B.1, they state right after equation D.7 that $0 \leq \bar{\sigma}_{i,h}^{-1} \leq 1 / \sqrt{H}$. Moreover, in the proof of Lemma B.5 the array of equations right after equation D.22, particularly $\left\|\bar{\sigma}_{i, h}^{-1} \phi\left(s_h^i, a_h^i\right)\right\|_2 \leq\left\|\phi\left(s_h^i, a_h^i\right)\right\|_2 / \sqrt{H}$, only holds true if this is $\sqrt{H}$.}

\newpage
\section{Proofs for Theorem \ref{alg:rappel}}
\label{app:hybridoffline}

The proof of Theorem \ref{thm:hybridoffline} follows from a series of distinct results, presented as three lemmas below. The first lemma demonstrates that RAPPEL achieves no higher error than LinPEVI-ADV+ itself, the second produces a $d_{\on}dH^4$ error bound, while the third produces a $d_{\on}^2dH^3$ error bound via a slightly different truncation argument. We will prove Equation~\ref{eqn:rappel-res1} in Lemma~\ref{lem:hybridoffline-instance}, which act as a general statistical guarantee for RAPPEL. We show the validity of the instance-dependent bound developed from Equation~\ref{eqn:rappel-res1} in Lemmas~\ref{lem:hybridoffline-H4} and~\ref{lem:hybridoffline-H3}.  We  observe that Theorem \ref{thm:hybridoffline} follows immediately after.


\begin{lem}[General Statistical Guarantee for RAPPEL, Algorithm \ref{alg:rappel}]
        \label{lem:hybridoffline-instance}
        For every $\delta \in (0,1)$ and any partition $\gX_{\off}, \gX_{\on}$, with probability at least $1-\delta$, RAPPEL achieves
        $$V_1^*(s)-V_1^{\widehat{\pi}}(s) \lesssim \sqrt{d}\sum_{h=1}^H \E_{\pi^*}||\phi(s_h, a_h)||_{(\Sigma_{\off,h}^{*}+ \Sigma_{\on,h}^*)^{-1}}\leq \sqrt{d}\sum_{h=1}^H \E_{\pi^*}||\phi(s_h, a_h)||_{\Sigma_{\off,h}^{*-1}}.$$
\end{lem}
\begin{proof}
    Before we proof the desired result, we first recall that 
    \begin{align}\label{equ:unweighted-cov}
        \Lambda_h&=\sum_{\tau \in \mathcal{D}} \phi\left(s_h^\tau, a_h^\tau\right) \phi\left(s_h^\tau, a_h^\tau\right)^{\top}+I_d,\\
        \label{equ:weighted-cov}
        \Sigma_h^*&=\sum_{\tau \in \mathcal{D}} \phi\left(s_h^\tau, a_h^\tau\right) \phi\left(s_h^\tau, a_h^\tau\right)^{\top} /\left[\mathbb{V}_h V_{h+1}^*\right]\left(s_h^\tau, a_h^\tau\right)+\lambda I_d.
    \end{align}
    Then, by invoking Theorem 2 from \cite{xiong2023nearly} with $N > \Omega(d^2H^6), \lambda=1/H^2, \beta_1 = O(\sqrt{d})$, we see that 
    \begin{align*}
        V_1^*(s)-V_1^{\widehat{\pi}}(s) 
       & \lesssim \sqrt{d} \sum_{h=1}^H \mathbb{E}_{\pi^*}\left[\|\phi(s_h, a_h)\|_{\Sigma_h^{*-1}} \mid s_1=s\right] \\
        & = \sqrt{d} \sum_{h=1}^H \mathbb{E}_{\pi^*}\left[\|\phi(s_h, a_h)\|_{(\Sigma_{\off,h}^{*}+ \Sigma_{\on,h}^*)^{-1}} \mid s_1=s\right],
    \end{align*}
    as $\Sigma_h = \Sigma_{\off,h}^{*}+ \Sigma_{\on,h}^*$. Noting that $\Sigma_{\on,h}^*$ is positive semi-definite, it then follows $\Sigma_{\off,h}^{*} \preceq \Sigma_{\off,h}^{*}+ \Sigma_{\on,h}^*$. Therefore, 
    $$\sqrt{d}\sum_{h=1}^H \E_{\pi^*}||\phi(s_h, a_h)||_{(\Sigma_{\off,h}^{*}+ \Sigma_{\on,h}^*)^{-1}}\leq \sqrt{d}\sum_{h=1}^H \E_{\pi^*}||\phi(s_h, a_h)||_{\Sigma_{\off,h}^{*-1}},$$
    and the inequality holds. 
\end{proof}

\begin{lem}[First Error Bound for RAPPEL, Algorithm \ref{alg:rappel}]
        \label{lem:hybridoffline-H4}
        For every $\delta \in (0,1)$ and any partition $\gX_{\off}, \gX_{\on}$, with probability at least $1-\delta$, RAPPEL achieves
        $$V_1^*(s)-V_1^{\widehat{\pi}}(s) \lesssim \sqrt{\frac{c_{\off}(\gX_{\off})dH^{4}}{N_{\off}}} + \sqrt{\frac{d_{\on}dH^{4}}{N_{\on}}}, \text{ where }$$
        $N \geq \max\left\{{\alpha_{\on}^{4}}{d_{\on}^{-4}},{\alpha_{\off}^{4}}{c_{\off}(\gX_{\off})^{-4}}\right\}\max\{N^*,\text{poly}(d, H,c_{\off}(\gX_{\off}), \log 1/\delta)\}$, where we define the quantities
        $\alpha_{\off} = \frac{N_{\off}}{N}$, $\alpha_{\on} = \frac{N_{\on}}{N}$, and the minimal samples for coverage is
        $$N^* = \min _N C \cdot N \text { s.t. } \inf _{\boldsymbol{\Lambda} \in \boldsymbol{\Omega}} \max _{\boldsymbol{\phi} \in \Phi} \boldsymbol{\phi}^{\top}\left(N(\boldsymbol{\Lambda}+\bar{\lambda} I)+\boldsymbol{\Lambda}_{\mathrm{off}}\right)^{-1} \boldsymbol{\phi} \leq {\tilde{O}(\max\{d_{\on}/N_{\on}, c_{\off}(\gX_{\off})/N_{\off}\})}.$$
\end{lem}


\begin{proof}
    
    Let $\gX_{\off}, \gX_{\on}$ be an arbitrary partition of $\gS \times \gA \times [H]$. Let us leave the choice of OPTCOV tolerance unspecified for the moment, and simply assume for now that we have data $\gD$ collected under the success event of Lemma \ref{lem:optcov-correct}. 
    
    We now invoke Theorem 2 from \cite{xiong2023nearly} on this dataset. As we choose $N > \Omega(d^2H^6)$, $\lambda=1/H^2$ and $\beta_1 = O(\sqrt{d})$, 
    we obtain the suboptimality gap decomposition below:
    \begin{align*}
        V_1^*(s)-V_1^{\widehat{\pi}}(s)
        \lesssim \sqrt{d} \sum_{h=1}^H \mathbb{E}_{\pi^*}\left[\|\phi(s_h, a_h)\|_{\Sigma_h^{*-1}} \mid s_1=s\right].
    \end{align*}

    This decomposition can be further decomposed into the sum of bonuses on the offline and online partitions $\gX_{\off}$ and $\gX_{\on}$, respectively:
    \begin{align*}
        & \sqrt{d} \sum_{h=1}^H \mathbb{E}_{\pi^*}\left[\|\phi(s_h, a_h)\|_{\Sigma_h^{*-1}} \mid s_1=s\right] \\
        &= \sqrt{d} \sum_{h=1}^H \left(\mathbb{E}_{\pi^*}\left[\|\phi(s_h, a_h)\|_{\Sigma_h^{*-1}} \mathbbm{1}_{\gX_{\on}}\mid s_1=s\right] + \mathbb{E}_{\pi^*}\left[\|\phi(s_h, a_h)\|_{\Sigma_h^{*-1}} \mathbbm{1}_{\gX_{\off}}\mid s_1=s\right]\right) \\
        &= \sqrt{d} \sum_{h=1}^H \mathbb{E}_{\pi^*}\left[\sqrt{\phi(s_h, a_h)^{\top}{\Sigma_h^{*-1}}\phi(s_h, a_h) }\mathbbm{1}_{\gX_{\on}}\mid s_1=s\right]\\
        &\qquad + \sqrt{d} \sum_{h=1}^H \mathbb{E}_{\pi^*}\left[\sqrt{\phi(s_h, a_h)^{\top}{\Sigma_h^{*-1}}\phi(s_h, a_h)}\mathbbm{1}_{\gX_{\off}}\mid s_1=s\right].
        \end{align*}

        We can further upper bound the above expectations under the optimal policy $\pi^*$ by taking the maximum of the quadratic form over each partition, yielding
        \begin{align*}
        & \sqrt{d} \sum_{h=1}^H \mathbb{E}_{\pi^*}\left[\|\phi(s_h, a_h)\|_{\Sigma_h^{*-1}} \mid s_1=s\right] \\
        &=\sqrt{d} \sum_{h=1}^H \max_{\phi_h \in \Phi_{\on}} \sqrt{\phi_h^{\top}\Sigma_h^{*-1}\phi_h}\mathbbm{1}_{\gX_{\on}} + \sqrt{d} \sum_{h=1}^H  \max_{\phi_h \in \Phi_{\off}} \sqrt{\phi_h^{\top}\Sigma_h^{*-1}\phi_h}\mathbbm{1}_{\gX_{\off}}\\
        &\leq \sqrt{d} \sum_{h=1}^H \max_{\phi_h \in \Phi_{\on}} \sqrt{\phi_h^{\top}\Sigma_h^{*-1}\phi_h} + \sqrt{d} \sum_{h=1}^H  \max_{\phi_h \in \Phi_{\off}}\sqrt{ \phi_h^{\top}\Sigma_h^{*-1}\phi_h}.
    \end{align*}
    
    From \cite{xiong2023nearly}, as $\left[\mathbb{V}_h V_{h+1}^*\right](\cdot, \cdot) \in\left[1, H^2\right]$, the weighted covariance matrix is uniformly upper bounded by the unweighted covariance matrix in the following manner: $$ \Sigma_h^{*-1} \preceq H^2 \Lambda_h^{-1},$$ 
    which leads to our conclusion that 
    $$V_1^*(s) - V_1^{\widehat{\pi}}(s) \lesssim \sqrt{d} \sum_{h=1}^H \max_{\phi_h \in \Phi_{\on}} \sqrt{H^2\phi_h^{\top}\Lambda_h^{-1}\phi_h} + \sqrt{d} \sum_{h=1}^H  \max_{\phi_h \in \Phi_{\off}}\sqrt{H^2 \phi_h^{\top}\Lambda_h^{-1}\phi_h}.$$
    
    We now further bound the above two quadratic forms over the online and offline partitions respectively. By Lemma \ref{lem:cov-linear}, the partial online coverage coefficient is bounded by the dimensionality of the online partition:
    $$c_{\on}(\gX_{\on}) = \inf_\pi \max_{\phi_h \in \Phi_{\on}} \phi_h^{\top}\E_{\bar\phi_h \sim d_h^\pi}[\bar\phi_h\bar\phi_h^{\top}]^{-1}\phi_h \leq d_{\on}.$$

    As we have $N_{\on}$ online episodes, the optimal covariates for online exploration would then yield
    $$\inf_{\boldsymbol{\Lambda}} \max_{\phi_h \in \Phi_{\on}} \phi_h^\top \boldsymbol{\Lambda}^{-1} \phi_h \lesssim c_{\on}(\gX_{\on})/N_{\on} \leq d_{\on}/N_{\on}.$$

    Conversely, we also have access to $N_{\off}$ episodes of offline data with the following guarantee that follows from an application of Matrix Chernoff:
    $$\max_{\phi_h \in \Phi_{\off}} \phi_h^\top \boldsymbol{\Lambda}_{\off}^{-1} \phi_h \lesssim c_{\off}(\gX_{\off})/N_{\off}.$$

    Therefore, by Lemma \ref{lem:cov-bound}, we can conclude that on its success event, running OPTCOV with tolerance $\tilde{O}(\max\{d_{\on}/N_{\on}, c_{\off}(\gX_{\off})/N_{\off}\}),$
    provides us covariates such that
    $$\max_{\phi_h \in \Phi} \phi_h^{\top}\boldsymbol{\Lambda}_{h}^{-1}\phi_h 
        \lesssim \max\left\{c_{\off}(\gX_{\off})/N_{\off}, d_{\on}/N_{\on}\right\},$$
    yielding the desired result.

    It now remains to work out the burn-in cost from running OPTCOV. The following quantity of the minimal online samples any algorithm requires to establish coverage was first proposed in \cite{wagenmaker2023leveraging}:
    $$N^* = \min _N C \cdot N \text { s.t. } \inf _{\boldsymbol{\Lambda} \in \boldsymbol{\Omega}} \max _{\phi \in \Phi} \boldsymbol{\phi}^{\top}\left(N(\boldsymbol{\Lambda}+\bar{\lambda} I)+\boldsymbol{\Lambda}_{\mathrm{off}}\right)^{-1} \boldsymbol{\phi} \leq \frac{\tilde{O}(\max\{d_{\on}/N_{\on}, c_{\off}(\gX_{\off})/N_{\off}\})}{6}.$$
    
    We can use this as follows. Invoking Lemma \ref{lem:optcov-correct}, we see that OPTCOV incurs 
    $$\max\left\{\left(\frac{N_{\off}}{c_{\off}(\gX_{\off})}\right)^{4/5}, \;\; \left(\frac{N_{\on}}{d_{\on}}\right)^{4/5}\right\}\max\{N^*,\text{poly}(d, H,c_{\off}(\gX_{\off}), \log 1/\delta)\}$$ episodes of online exploration, for an overall burn-in cost of
    $$N_{\off} + N_{\on} \geq \max\left\{\frac{\alpha_{\on}^4}{d_{\on}^4},\;\;\frac{\alpha_{\off}^4}{c_{\off}(\gX_{\off})^4}\right\}\max\{N^*,\text{poly}(d, H,c_{\off}(\gX_{\off}), \log 1/\delta)\}$$
    episodes, where $\alpha_{\off} = \frac{N_{\off}}{N_{\off} + N_{\on}}$ and $\alpha_{\on} = \frac{N_{\on}}{N_{\off} + N_{\on}}$. 
    
    Note that the more even the proportion of offline to online samples, the smaller $\alpha_{\off}, \alpha_{\on}$ are. In fact, as $\alpha_{\off}^4, \alpha_{\on}^4 \in [0.0625,1]$, this term contributes no more than a constant factor that is no greater than $1$ to the final sample complexity.

    We then have that 
    $$V_1^*(s)-V_1^{\widehat{\pi}}(s) \lesssim \inf_{\gX_{\off}, \gX_{\on}} \left(\sqrt{\frac{c_{\off}(\gX_{\off})dH^4}{N_{\off}}} + \sqrt{\frac{d_{\on}dH^4}{N_{\on}}}\right)$$
    with probability at least $1-\delta$, when $N \geq \max\left\{\frac{\alpha_{\on}^{4}}{d_{\on}^{4}},\;\;\frac{\alpha_{\off}^{4}}{c_{\off}(\gX_{\off})^{4}}\right\}\max\{N^*,\text{poly}(d, H,c_{\off}(\gX_{\off}), \log 1/\delta)\}$.
\end{proof}

\begin{lem}[Second Error Bound for RAPPEL, Algorithm \ref{alg:rappel}]
        \label{lem:hybridoffline-H3}
        For every $\delta \in (0,1)$ and any partition $\gX_{\off}, \gX_{\on}$, with probability at least $1-\delta$, RAPPEL achieves
        $$V_1^*(s)-V_1^{\widehat{\pi}}(s) \lesssim \sqrt{\frac{c_{\off}(\gX_{\off})^2dH^3}{N_{\off}\alpha_{\off}}} + \sqrt{\frac{d_{\on}^2dH^3}{N_{\on}\alpha_{\on}}}, \text{ where }$$
        $N \geq \max\left\{{\alpha_{\on}^{4}}{d_{\on}^{-4}},{\alpha_{\off}^{4}}{c_{\off}(\gX_{\off})^{-4}}\right\}\max\{N^*,\text{poly}(d, H,c_{\off}(\gX_{\off}), \log 1/\delta)\}$, we define the quantities
        $\alpha_{\off} = \frac{N_{\off}}{N}$, $\alpha_{\on} = \frac{N_{\on}}{N}$, and the minimal samples for coverage is
        $$N^* = \min _N C \cdot N \text { s.t. } \inf _{\boldsymbol{\Lambda} \in \boldsymbol{\Omega}} \max _{\boldsymbol{\phi} \in \Phi} \boldsymbol{\phi}^{\top}\left(N(\boldsymbol{\Lambda}+\bar{\lambda} I)+\boldsymbol{\Lambda}_{\mathrm{off}}\right)^{-1} \boldsymbol{\phi} \leq {\tilde{O}(\max\{d_{\on}/N_{\on}, c_{\off}(\gX_{\off})/N_{\off}\})}.$$
\end{lem}
\begin{proof}

    First, we set up some preliminaries. Following the same argument as the proof of Lemma \ref{lem:hybridoffline-H4}, we can establish that, for arbitrary partition $\gX = \gX_{\on}\cup\gX_{\off}$, we have
    $$c_{\on}(\gX_{\on}) \leq d_{\on},$$
    and running OPTCOV with tolerance $\tilde{O}(\max\{d_{\on}/N_{\on}, c_{\off}(\gX_{\off})/N_{\off}\}),$
    yields:
    $$\max_{\phi_h \in \Phi} \phi_h^{\top}\Lambda_{h}^{-1}\phi_h 
        \lesssim \max\left\{c_{\off}(\gX_{\off})/N_{\off}, d_{\on}/N_{\on}\right\}.$$
    This incurs
    $$\max\left\{\left(\frac{N_{\off}}{c_{\off}(\gX_{\off})}\right)^{4/5}, \;\; \left(\frac{N_{\on}}{d_{\on}}\right)^{4/5}\right\}\max\{N^*,\text{poly}(d, H,c_{\off}(\gX_{\off}), \log 1/\delta)\}$$ episodes of online exploration, for an overall burn-in cost of
    $$N_{\off} + N_{\on} \geq \max\left\{\frac{\alpha_{\on}^4}{d_{\on}^4},\;\;\frac{\alpha_{\off}^4}{c_{\off}(\gX_{\off})^4}\right\}\max\{N^*,\text{poly}(d, H,c_{\off}(\gX_{\off}), \log 1/\delta)\}$$
    episodes.
    
    
    To tighten the horizon dependence even further from the result of Lemma \ref{lem:hybridoffline-H4}, we turn to the total variance lemma. 
    i.e. Lemma C.5 in \cite{jin2018q}, indicating that
    $$\frac{1}{NH}\sum_{\tau \in \mathcal{D}}\sum_{h=1}^H \left[\mathbb{V}_h V_{h+1}^*\right]\left(s_h^\tau, a_h^\tau\right) \lesssim \tilde{O}\left(H+\frac{H^2}{N}\right).$$
    Then, we directly apply Lemma~\ref{lem:total-variance-concentration} with $\gamma = \max\left\{d_{\on}/N_{\on},c_{\off}(\gX_{\off})/N_{\off}\right\}$ and $\bar{\sigma} = H + H^2/N$, we will then obtain that

     \begin{align}
         \sum_{h=1}^{H}\max_{\phi_h\in\Phi}\sqrt{\phi_h^{\top}{\Sigma_h^{\star}}^{-1}\phi_h}
         & \leq \bigg(\frac{d_{\on}}{N_{\on}}+\frac{c_{\off}(\gX_{\off})}{N_{\off}}\bigg)H\sqrt{N\bigg(H+\frac{H^2}{N}\bigg)}\nonumber\\
         &\leq \bigg(\frac{d_{\on}}{N_{\on}}+\frac{c_{\off}(\gX_{\off})}{N_{\off}}\bigg)\sqrt{NH^3 + H^4}\nonumber\\
            &\leq \sqrt{\frac{c_{\off}(\gX_{\off})^2H^3}{N_{\off}\alpha_{\off}} + \frac{c_{\off}(\gX_{\off})^2 H^4}{N_{\off}^2}} + \sqrt{\frac{d_{\on}^2H^3}{N_{\on}\alpha_{\on}} + \frac{d_{\on}^2 H^4}{N_{\on}^2}} \\
     &\lesssim \sqrt{\frac{c_{\off}(\gX_{\off})^2H^3}{N_{\off}\alpha_{\off}}} + \sqrt{\frac{d_{\on}^2H^3}{N_{\on}\alpha_{\on}}},
     \end{align}
     which leads to our final result:
     $$V_1^*(s)-V_1^{\widehat{\pi}}(s) \lesssim \inf_{\gX_{\off}, \gX_{\on}} \left(\sqrt{\frac{c_{\off}(\gX_{\off})^2dH^3}{N_{\off}\alpha_{\off}}} + \sqrt{\frac{d_{\on}^2dH^3}{N_{\on}\alpha_{\on}}}\right),$$
     where $\alpha_{\off} = N_{\off}/N$ and $\alpha_{\on} = N_{\on}/N$.
\end{proof}
\section{Proof of Corollary~\ref{cor:rappel-tabular}}
\label{app:proof-cor-rappel-tab}
\begin{proof}
    In tabular case, we set $\phi(s,a)={\bm 1}_{s,a}$ and $d=|\mathcal{S}|\cdot |\mathcal{A}|$. Let $N_h(s,a)$ be the number of visits to a specific state-action pair $(s,a,h)$. As the exploration algorithm OPTCOV ensures that 
    $$\max_{s,a,h} \frac{1}{N_h(s,a)}\leq \max\left(\frac{d_{\on}}{N_{\on}},\frac{c_{\off}(\gX_{\off})}{N_{\off}}\right),$$
    we bound the error in the following way follows from Lemma~\ref{lem:hybridoffline-instance},
    \begin{align}
      V_1^*(s)-V_1^{\widehat{\pi}}(s) & \lesssim \sqrt{d}\sum_{h=1}^H \E_{\pi^*}||\phi(s_h, a_h)||_{(\Sigma_{\off,h}^{*}+ \Sigma_{\on,h}^*)^{-1}}\nonumber\\
      & \leq \sqrt{|\mathcal{S}||\mathcal{A}|}\sum_{h=1}^{H}\sum_{s,a}d_{h}^{\star}(s,a)\sqrt{\frac{\left[\mathbb{V}_h V_{h+1}^*\right]\left(s, a\right)}{N_h(s,a)}},\nonumber
    \end{align}
    where the last inequality follows from the fact that ${\Sigma_h^{\star}}=\text{diag}\big(N_h(s,a)/\left[\mathbb{V}_h V_{h+1}^*\right](s,a)\big)_{s\in\mathcal{S},a\in\mathcal{A}}$. We will then decompose the state-action space into $\gX_{\off}$ and $\gX_{\on}$, and bound the two parts seperately based on the tolerance level of OPTCOV,
    \begin{align}
       V_1^*(s)-V_1^{\widehat{\pi}}(s) & \lesssim\sqrt{|\mathcal{S}||\mathcal{A}|}\sum_{h=1}^{H}\sum_{s,a}d_{h}^{\star}(s,a)\sqrt{\frac{\left[\mathbb{V}_h V_{h+1}^*\right]\left(s, a\right)}{N_h(s,a)}}\mathbbm{1}_{\gX_{\off}} \nonumber\\
      &\qquad + \sqrt{|\mathcal{S}||\mathcal{A}|}\sum_{h=1}^{H}\sum_{s,a}d_{h}^{\star}(s,a)\sqrt{\frac{\left[\mathbb{V}_h V_{h+1}^*\right]\left(s, a\right)}{N_h(s,a)}}\mathbbm{1}_{\gX_{\on}}\nonumber\\
      &\leq \sqrt{\frac{|\mathcal{S}||\mathcal{A}|c_{\off}(\gX_{\off})}{N_{\off}}}\sum_{h=1}^{H}\sum_{s,a}d_{h}^{\star}(s,a)\sqrt{\left[\mathbb{V}_h V_{h+1}^*\right]\left(s, a\right)}\mathbbm{1}_{\gX_{\off}}\nonumber\\
      &\qquad + \sqrt{\frac{|\mathcal{S}||\mathcal{A}|d_{\on}}{N_{\on}}}\sum_{h=1}^{H}\sum_{s,a}d_{h}^{\star}(s,a)\sqrt{\left[\mathbb{V}_h V_{h+1}^*\right]\left(s, a\right)}\mathbbm{1}_{\gX_{\on}}\nonumber\\
      &\leq \sqrt{|\mathcal{S}||\mathcal{A}|}\left(\sqrt{\frac{c_{\off}(\gX_{\off})}{N_{\off}}} + \sqrt{\frac{d_{\on}}{N_{\on}}}\right) \sum_{h=1}^{H}\sum_{s,a}\sqrt{d_{h}^{\star}(s,a)\left[\mathbb{V}_h V_{h+1}^*\right]\left(s, a\right)}.\nonumber
\end{align}
As the optimal policy $\pi^{\star}$ executes a deterministic action $\pi^{\star}(s)$ for any state $s$, the inequality can be further bounded as
\begin{align}
      V_1^*(s)-V_1^{\widehat{\pi}}(s)  &  \lesssim \sqrt{|\mathcal{S}||\mathcal{A}|}\left(\sqrt{\frac{c_{\off}(\gX_{\off})}{N_{\off}}} + \sqrt{\frac{d_{\on}}{N_{\on}}}\right) \sum_{h=1}^{H}\sum_{s}\sqrt{d_{h}^{\star}(s,\pi^{\star}(s))\left[\mathbb{V}_h V_{h+1}^*\right]\left(s, \pi^{\star}(s)\right)}\nonumber\\
      &\leq \sqrt{H|\mathcal{S}|^2|\mathcal{A}|}\left(\sqrt{\frac{c_{\off}(\gX_{\off})}{N_{\off}}} + \sqrt{\frac{d_{\on}}{N_{\on}}}\right) \sqrt{\sum_{h=1}^{H}\sum_{s}d_{h}^{\star}(s,\pi^{\star}(s))\left[\mathbb{V}_h V_{h+1}^*\right]\left(s, \pi^{\star}(s)\right)}\nonumber\\
      &\leq \sqrt{H|\mathcal{S}|^2|\mathcal{A}|}\left(\sqrt{\frac{c_{\off}(\gX_{\off})}{N_{\off}}} + \sqrt{\frac{d_{\on}}{N_{\on}}}\right)\sqrt{\sum_{h=1}^{H}\mathbb{E}_{(s,a)\sim d_{\pi^{\star}}}\left[\mathbb{V}_h V_{h+1}^*\right]\left(s, a\right)}\nonumber\\
      &\leq\sqrt{H^3|\mathcal{S}|^2|\mathcal{A}|}\left(\sqrt{\frac{c_{\off}(\gX_{\off})}{N_{\off}}} + \sqrt{\frac{d_{\on}}{N_{\on}}}\right),
    \end{align}
    where the last inequality follows from the proof of Lemma C.5. in~\cite{jin2018q}.
\end{proof}
\section{On concentrability and coverability}
\label{app:coverability}
    
{\bf Lemma \ref{lem:cov-linear}.}{\it \quad 
    For any partition $\gX_{\off}, \gX_{\on}$, we have that $c_{\on}(\gX_{\on}) \leq d_{\on}$. Similarly, there exists a partition such that $c_{\off}(\gX_{\off}) = O(d)$. 
}

\begin{proof}
    This proof follows a similar strategy to that of Lemma B.10 in \cite{wagenmaker2023instancedependent}, except that we exploit the projections onto $d_{\on}$ to get a bound that depends on $d_{\on} \leq d$, instead of $d$. We wish to bound 
    $$c_{\on}(\gX_{\on}) = \inf_\pi \max_h\frac{1}{\lambda_{d_{\on}}(\E_{d^\pi_h}[(\gP_{\on}\phi_h)(\gP_{\on}\phi_h)^{\top}])}.$$

    $\gP_{\on} \in \R^{d\times d}$ has rank $d_{\on} \leq d$, so we can decompose this with the thin SVD into $\gP_{\on} = U_{\on}U_{\on}^{\top}$, where $U_{\on} \in \R^{d \times d_{\on}}$. It then holds that
    $$\lambda_{d_{\on}}(\E_{d^\pi_h}[(\gP_{\on}\phi_h)(\gP_{\on}\phi_h)^{\top}]) = \lambda_{\min}(\E_{d^\pi_h}[(U_{\on}^{\top}\phi_h)(U_{\on}^{\top}\phi_h)^{\top}]),$$
    and from Lemma \ref{lem:hole-digging} that 
    $$c_{\on}(\gX_{\on}) = \inf_\pi \sup_{v_h \in \Phi_{\on}} v_h^{\top}U_{\on}E_{d^\pi_h}[(U_{\on}^{\top}\phi_h)(U_{\on}^{\top}\phi_h)^{\top}]^{-1} U_{\on}^{\top}v_h.$$


    Apply Jensen's inequality to find that for any $v_h \in \Phi_{\on}$,
     $$v_h^{\top}U_{\on}E_{d^\pi_h}[(U_{\on}^{\top}\phi_h)(U_{\on}^{\top}\phi_h)^{\top}] U_{\on}^{\top}v_h \geq v_h^{\top}U_{\on}\E_{\phi_h \sim d^\pi_h}[U_{\on}^{\top}\phi_h]\E_{\phi_h \sim d^\pi_h}[U_{\on}^{\top}\phi_h]^{\top}U_{\on}^{\top}v_h.$$

    Then, we can bound
    \begin{align*}
        c_{\on}(\gX_{\on}) 
        &= \inf_\pi \sup_{v_h \in \Phi_{\on}} v_h^{\top}U_{\on}E_{d^\pi_h}[(U_{\on}^{\top}\phi_h)(U_{\on}^{\top}\phi_h)^{\top}]^{-1} U_{\on}^{\top}v_h \\
        &\leq \inf_{\rho} \sup_{v_h \in \Phi_{\on}} v_h^{\top}U_{\on} \left(\E_{\pi \sim \rho} \left[\E_{\phi_h \sim d^\pi_h}[U_{\on}^{\top}\phi_h]E_{\phi_h \sim d^\pi_h}[U_{\on}^{\top}\phi_h^{\top}]\right]\right)^{-1}U_{\on}^{\top}v_h.
    \end{align*}
    By Kiefer-Wolfowitz \citep{lattimore2020learning}, this is bounded by $d_{\on}$.

    Similarly, 
    \begin{align*}
        \inf_{\gX_{\off}, \gX_{\on}} c_{\off}(\gX_{\off})
        &= \inf_{\gX_{\off}, \gX_{\on}} \max_h\frac{1}{\lambda_{d_{\off}}(\E_{\mu_h}[(\gP_{\off}\phi_h)(\gP_{\off}\phi_h)^{\top}])} \\
        &= \inf_{\gX_{\off}, \gX_{\on}} \max_h\frac{1}{\lambda_{\min}(\E_{\mu_h}[(U^{\top}_{\off}\phi_h)(U^{\top}_{\off}\phi_h)^{\top}])} \\
        &\leq O(d).
    \end{align*}
    where the upper bound is achieved when, for instance, we choose $\gX_{\off}$ such that $\Phi_{\off} = \text{Span}\left((v_{h,1},...,v_{h,k_h})_{h\in[H]}\right)$, where $v_{h,i}$ is the $i$-th largest eigenvector of $\E_{\mu}[\phi_h\phi_h^{\top}]\approx \frac{1}{N_{\off}}\sum_{\tau \in \gD_{\off}} \phi_h(s_h^\tau, a_h^\tau)\phi_h(s_h^\tau, a_h^\tau)^{\top}$, and $v_{h,k_h}$ is the eigenvector corresponding to the largest eigenvalue $\lambda_{h, k_h} \geq \Omega(1/k_h)$. The largest eigenvalue $\lambda_{h,1}$ is always $\Omega(1/d)$ for non-null features, so there always exists such a partition where $d_{\off}$ is at least 1.

\end{proof}

    Informally, one can choose the offline partition to be the span of the large eigenvectors of the covariance matrix, so the smallest eigenvalue of the projected covariance matrix, i.e. the partial all policy concentrability coefficient, is no larger than the dimension of the partition.

\begin{lem}[Maximum Eigenvalue Bound with OPTCOV]
    \label{lem:cov-bound}
    On any partition $\gX_{\off}, \gX_{\on}$, if we run OPTCOV with tolerance $\tilde{O}(\max\{d_{\on}/N_{\on}, c_{\off}(\gX_{\off})/N_{\off}\}),$
    on this partition we also have that
    $$\max_{\phi_h \in \Phi} \phi_h^{\top}\Lambda_{h}^{-1}\phi_h 
        \lesssim \max\left\{c_{\off}(\gX_{\off})/N_{\off}, d_{\on}/N_{\on}\right\}.$$
\end{lem}

\begin{proof}
    By Lemma \ref{lem:cov-linear}, for any partition, we have that
    $$c_{\on}(\gX_{\on}) = \inf_\pi \max_{\phi_h \in \Phi_{\on}} \phi_h^{\top}\E_{\bar\phi_h \sim d_h^\pi}[\bar\phi_h\bar\phi_h^{\top}]^{-1}\phi_h \leq d_{\on},$$

    Applying Matrix Chernoff, we have that with probability at least $1-\delta$, 
    $$\max_{\phi_h \in \Phi_{\off}} \phi_h^{\top}\Lambda_{h,\off}^{-1}\phi_h \leq \max_{\phi_h \in \Phi_{\off}} \phi_h^{\top}\E_{\bar\phi_h \sim \mu_h}[\bar\phi_h\bar\phi_h^{\top}+N_{\off}^{-1}\mathbf{I}]^{-1}\phi_h N^{-1}_{\off} \left(1-\sqrt{\frac{2}{N_{\text {off }}} \log \left(\frac{4d}{\delta}\right)}\right)^{-1},$$
    and similarly for $c_{\on}(\gX_{\on})$ we also have that  
    $$\inf_\pi \max_{\phi_h \in \Phi_{\on}} \phi_h^{\top}\Lambda_{h,\pi}^{-1}\phi_h \leq \inf_\pi \max_{\phi_h \in \Phi_{\on}} \phi_h^{\top}\E_{\bar\phi_h \sim \mu_h}[\bar\phi_h\bar\phi_h^{\top}]^{-1}\phi_h N^{-1}_{\on} \left(1-\sqrt{\frac{2}{N_{\on}} \log \left(\frac{4d}{\delta}\right)}\right)^{-1}.$$

    As $\Lambda_{h,\off} + \Lambda_{h,\on} = \Lambda_h$, we have
    \begin{align*}
        \max_{\phi_h \in \Phi} \phi_h^{\top}\Lambda_{h}^{-1}\phi_h 
        &= \max\left\{\max_{\phi_h \in \Phi_{\off}} \phi_h^{\top}\Lambda_{h}^{-1}\phi_h, \max_{\phi_h \in \Phi_{\on}} \phi_h^{\top}\Lambda_{h}^{-1}\phi_h \right\}\\
        &\lesssim \max\left\{c_{\off}(\gX_{\off})/N_{\off}, \max_{\phi_h \in \Phi_{\on}} \phi_h^{\top}\Lambda_{h}^{-1}\phi_h\right\}\nonumber,
    \end{align*} 
    where the last step follows from the choice of partition. So it suffices to run OPTCOV with tolerance $\tilde{O}(\max\{d_{\on}/N_{\on}, c_{\off}(\gX_{\off})/N_{\off}\}),$

    to find that there exists at least one partition such that
    $$\max_{\phi_h \in \Phi} \phi_h^{\top}\Lambda_{h}^{-1}\phi_h 
        \lesssim \max\left\{c_{\off}(\gX_{\off})/N_{\off}, d_{\on}/N_{\on}\right\}.$$
\end{proof}

\begin{lem}[Coverability Coefficient Is Bounded In Tabular MDPs]
    \label{lem:cov-tabular}
    If the underlying MDP is tabular, for any partition $\gX_{\off}, \gX_{\on}$, we have that $c_{\on}(\gX_{\on}) \leq d_{\on}$.
\end{lem}
\begin{proof}
    First, we write the concentrability coefficient in terms of densities.
    \begin{align*}
        c_{\on}(\gX_{\on}) 
        &= \min_\pi \max_h\frac{1}{\lambda_{d_{\on}}(\E_{d^\pi_h}[(\gP_{\on}\phi_h)(\gP_{\on}\phi_h)^{\top}])} \\
        &\leq \min_\pi \max_h\frac{\mathbbm{1}_{\gX_{on}}}{\min_{s,a}  d_h^\pi(s,a)\mathbbm{1}_{\gX_{on}}}\\
        &\leq \min_\pi \max_{h,s,a} \frac{\mathbbm{1}_{\gX_{on}}}{d_h^\pi(s,a)\mathbbm{1}_{\gX_{on}}}.
    \end{align*}

    By the same trick that \cite{xie2022role} use in their Lemma 3, 
    \begin{align*}
        \frac{\mathbbm{1}_{\gX_{\on}}}{d_h^{\pi}(s,a)\mathbbm{1}_{\gX_{\on}}}
        &\leq \frac{\mathbbm{1}_{\gX_{\on}}}{\sup_{\pi{''}} d_h^{\pi''}(s, a)\mathbbm{1}_{\gX_{\on}} / \sum_{s{'}, a{'}} \sup_{\pi{'} } d_h^{\pi'}\left(s', a'\right)\mathbbm{1}_{\gX_{\on}}}\\
        &\leq \frac{\sum_{s, a} \sup_{\pi} d_h^{\pi}\left(s, a\right)\mathbbm{1}_{\gX_{\on}}}{\sup_{\pi} d_h^{\pi}\left(s, a\right)\mathbbm{1}_{\gX_{\on}}} \\
        &\leq d_{\on}.
    \end{align*}
    
\end{proof}

\section{Proofs for Algorithm \ref{alg:hyrule}}

\label{app:hybridonline}



\subsection{Setup}

We consider the same state-action space splitting framework of \cite{tan2024natural}.  Let $\gX_{\on} \cup \gX_{\off} = [H]\times \gS \times \gA$. Then, their images under the feature map $\Phi_{\off} = \text{Span}(\phi(\gX_{\off,h}))_{h\in [H]} \subseteq \R^d$ and $\Phi_{\on} = \text{Span}(\phi(\gX_{\on,h}))_{h\in [H]}\subseteq \R^d$ are subspaces of $\gX$ with dimension $d_{\off}$ and $d_{\on}$, respectively. We denote $\gP_{\off}, \gP_{\on}$ as the orthogonal projection operators onto these subspaces respectively. The partial offline all-policy concentrability coefficient 
$$c_{\off}(\gX_{\off}) = \max_h\frac{1}{\lambda_{d_{\off}}(\E_{\mu_h}[(\gP_{\off}\phi_h)(\gP_{\off}\phi_h)^{\top}])},$$
is bounded by the inverse of the $d_{\off}$-th largest eigenvalue of the covariance matrix of the projected feature maps onto the offline partition, where $\lambda_k$ is the $k$-th largest eigenvalue. Write $\mathbbm{1}_{\gX_{\on}}$ as shorthand for $\mathbbm{1}((s,a,h) \in \gX_{\on})$, and similarly for $\mathbbm{1}_{\gX_{\off}}$.

Now, we work through the analysis of \cite{he2023nearly} to ensure that their result holds in our setting, where the regret decomposes into online part $\|\boldsymbol{\Sigma}_{t,h}^{-1/2}\phi_h(s_h^{(t)},a_h^{(t)})\mathbbm{1}_{\gX_{\on}}\|_2$ and offline part $\|\boldsymbol{\Sigma}_{t,h}^{-1/2}\phi_h(s_h^{(t)},a_h^{(t)})\mathbbm{1}_{\gX_{\off}}\|_2$ respectively, instead of $\|\boldsymbol{\Sigma}_{t,h}^{-1/2}\phi_h(s_h^{(t)},a_h^{(t)})\|_2.$

\subsection{High-probability events}
We define several ``high probability" events which are similar to those defined in \cite{he2023nearly}.
\begin{itemize}
  \item We define $\widetilde{w}_{t,h}$ as the solution of the weighted ridge regression problem for the squared value function
  \begin{align}
  \widetilde{w}_{t,h} = \mathbf{\Sigma}_{t, h}^{-1} \sum_{i=1}^{t-1} \bar{\sigma}_{i, h}^{-2} \boldsymbol{\phi}(s_h^{(i)}, a_h^{(i)}) V^2_{t, h+1}(s_{h+1}^{(i)}).
  \end{align}
  \item We define $\mathcal{E}$ as the event where the following inequalities hold for all $s, a, t, h \in \mathcal{S} \times \mathcal{A} \times[T] \times[H]$:

\begin{align}
\left|\widehat{\mathbf{w}}_{t, h}^{\top} \boldsymbol{\phi}(s, a)-\left[\mathbb{P}_h V_{t, h+1}\right](s, a)\right| &\leq \bar{\beta} \sqrt{\boldsymbol{\phi}(s, a)^{\top} \boldsymbol{\Sigma}_{t, h}^{-1} \boldsymbol{\phi}(s, a)}, \label{equ:coarse-1}\\
\left|\widetilde{\mathbf{w}}_{t, h}^{\top} \boldsymbol{\phi}(s, a)-\left[\mathbb{P}_h V_{t, h+1}^2\right](s, a)\right| &\leq \widetilde{\beta} \sqrt{\boldsymbol{\phi}(s, a)^{\top} \boldsymbol{\Sigma}_{t, h}^{-1} \boldsymbol{\phi}(s, a)}, \label{equ:coarse-2}\\
\left|\widecheck{\mathbf{w}}_{t, h}^{\top} \boldsymbol{\phi}(s, a)-\left[\mathbb{P}_h \widecheck{V}_{t, h+1}\right](s, a)\right| &\leq \bar{\beta} \sqrt{\boldsymbol{\phi}(s, a)^{\top} \boldsymbol{\Sigma}_{t, h}^{-1} \boldsymbol{\phi}(s, a)},\label{equ:coarse-3}
\end{align}
$$\widetilde{\beta}=O\left(H^2 \sqrt{d \lambda}+\sqrt{d^3 H^4 \log ^2(d H N /(\delta \lambda))}\right), \bar{\beta}=O\left(H \sqrt{d \lambda}+\sqrt{d^3 H^2 \log ^2(d H N /(\delta \lambda))}\right).$$
This is the ``coarse event'' as mentioned in their paper, where concentration holds for the value and squared value function with all three estimators.

\item We define $\widetilde{\mathcal{E}}_h$ as the event that for all episodes $t \in[T]$, stages $h \leq h^{\prime} \leq H$ and state-action pairs $(s, a) \in \mathcal{S} \times \mathcal{A}$, the weight vector $\widehat{\mathbf{w}}_{t, h}$ satisfies
\begin{align}\label{equ:weight-concentration}
  \left|\widehat{\mathbf{w}}_{t, h^{\prime}}^{\top} \boldsymbol{\phi}(s, a)-\left[\mathbb{P}_h V_{t, h^{\prime}+1}\right](s, a)\right| \leq \beta \sqrt{\boldsymbol{\phi}(s, a)^{\top} \boldsymbol{\Sigma}_{t, h^{\prime}}^{-1} \boldsymbol{\phi}(s, a)},
\end{align}
where 
$$\beta=O\left(H \sqrt{d \lambda}+\sqrt{d \log ^2(1+d N H /(\delta \lambda))}\right).$$ Furthermore,  let $\widetilde{\mathcal{E}}=\widetilde{\mathcal{E}}_1$ denotes the event that~(\ref{equ:weight-concentration}) holds for all stages $h \in[H]$. This is the fine event where concentration for $\bf{\hat{w}}$ is tighter than that required in~(\ref{equ:coarse-1}) to~(\ref{equ:coarse-3}).
\end{itemize}

Equipped with these definitions, we recall the following lemmas from~\cite{he2023nearly}:
\begin{lem}[Lemma B.1, \cite{he2023nearly}]
    $\mathcal{E} \text { holds with probability at least } 1-7 \delta$.
    \label{lem:he-B1-concentration}
\end{lem}

\begin{lem}[Lemma B.2, \cite{he2023nearly}]
    On the event $\mathcal{E}$ and $\widetilde{\mathcal{E}}_{h+1}$, for each episode $t \in[T]$ and stage $h$, the estimated variance satisfies
$$
\begin{aligned}
& \left|\left[\overline{\mathbb{V}}_h V_{t, h+1}\right]\left(s_h^{(t)}, a_h^{(t)}\right)-\left[\mathbb{V}_h V_{t, h+1}\right]\left(s_h^{(t)}, a_h^{(t)}\right)\right| \leq E_{t, h}, \\
& \left|\left[\overline{\mathbb{V}}_h V_{t, h+1}\right]\left(s_h^{(t)}, a_h^{(t)}\right)-\left[\mathbb{V}_h V_{h+1}^*\right]\left(s_h^{(t)}, a_h^{(t)}\right)\right| \leq E_{t, h}+D_{t, h} .
\end{aligned}
$$
\label{lem:he-B2}
\end{lem}

\begin{lem}[Lemma B.3, \cite{he2023nearly}]
    On the event $\mathcal{E}$ and $\widetilde{\mathcal{E}}_{h+1}$, for any episode $t$ and $i>t$, we have
$$
\left[\mathbb{V}_h\left(V_{i, h+1}-V_{h+1}^*\right)\right]\left(s_h^{(t)}, a_h^{(t)}\right) \leq D_{t, h} /\left(d^3 H\right) .
$$
\label{lem:he-B3}
\end{lem}

\begin{lem}[Lemma B.4, \cite{he2023nearly}]
    On the event $\mathcal{E}$ and $\widetilde{\mathcal{E}}_h$, for all episodes $t \in [T]$ and stages $h \leq h^{\prime} \leq H$, we have $Q_{t, h}(s, a) \geq Q_h^*(s, a) \geq$ $\widecheck{Q}_{t, h}(s, a)$. In addition, we have $V_{t, h}(s) \geq V_h^*(s) \geq \widecheck{V}_{t, h}(s)$.
    \label{lem:he-B4-optimism}
\end{lem}

\begin{lem}[Lemma B.5, \cite{he2023nearly}]
    $\text {  On event }\mathcal{E} \text {, event } \widetilde{\mathcal{E}} \text { holds with probability at least } 1-\delta \text {. }$
\end{lem}

\subsection{Regret decomposition}
From \cite{he2023nearly}, based on Lemma B.4 of their paper, $Q_{t,h}(s_h^{(t)}, a_h^{(t)}) = V_{t,h}(s_h^{(t)}) \geq V_h^*(s_h^{(t)})$, i.e. optimism holds for all episodes and timesteps. Therefore,
\begin{align*}
    \text{Reg}(T) 
    &\lesssim\sum_{t=1}^{T} \sum_{h=1}^H\left\{\left[\mathbb{P}_h\left(V_{t, h+1}-V_{t, h+1}^{\pi^{(t)}}\right)\right]\left(s_h^{(t)}, a_h^{(t)}\right)-\left(V_{t, h+1}\left(s_{h+1}^{(t)}\right)-V_{t, h+1}^{\pi^{(t)}}\left(s_{h+1}^{(t)}\right)\right)\right\} \\
    &+ \beta\sum_{t=1}^{T}\sum_{h=1}^H\|\boldsymbol{\Sigma}_{t,h}^{-1/2}\phi_h(s_h^{(t)},a_h^{(t)})\|_2.
\end{align*}

Accordingly, given a partition $\gX_{\off}, \gX_{\on}$ of $[H]\times\gS\times\gA$, we can further decompose this into the fraction of episodes where each partition is visited,$$ \sum_{t=1}^{T}\sum_{h=1}^H\|\boldsymbol{\Sigma}_{t,h}^{-1/2}\phi_h(s_h^{(t)},a_h^{(t)})\|_2 = \sum_{h,t}\|\boldsymbol{\Sigma}_{t,h}^{-1/2}\phi_h(s_h^{(t)},a_h^{(t)})\mathbbm{1}_{\gX_{\off}}\|_2 + \sum_{h,t}\|\boldsymbol{\Sigma}_{t,h}^{-1/2}\phi_h(s_h^{(t)},a_h^{(t)})\mathbbm{1}_{\gX_{\on}}\|_2.$$

\cite{he2023nearly} define the events

\begin{align*}
\mathcal{E}_1=\left\{\forall h \in[H], \sum_{t=1}^{T}\right. & \sum_{h^{\prime}=h}^H\left[\mathbb{P}_h\left(V_{t, h+1}-V_{t, h+1}^{\pi^{(t)}}\right)\right]\left(s_h^{(t)}, a_h^{(t)}\right) \\
& \left.-\sum_{t=1}^{T} \sum_{h^{\prime}=h}^H\left(V_{t, h+1}\left(s_{h+1}^{(t)}\right)-V_{t, h+1}^{\pi^{(t)}}\left(s_{h+1}^{(t)}\right)\right) \leq 2 \sqrt{2 H^3 T \log (H / \delta)}\right\},
\end{align*}
\begin{align*}
\mathcal{E}_2=\left\{\forall h \in[H], \sum_{t=1}^{T}\right. & \sum_{h^{\prime}=h}^H\left[\mathbb{P}_h\left(V_{t, h+1}-\widecheck{V}_{t, h+1}\right)\right]\left(s_h^{(t)}, a_h^{(t)}\right) \\
& \left.-\sum_{t=1}^{T} \sum_{h^{\prime}=h}^H\left(V_{t, h+1}\left(s_{h+1}^{(t)}\right)-\widecheck{V}_{t, h+1}\left(s_{h+1}^{(t)}\right)\right) \leq 2 \sqrt{2 H^3 T \log (H / \delta)}\right\},
\end{align*}
which they show that by Azuma-Hoeffding, both hold with probability $1-\delta$ each. As such, we have that
$$\text{Reg}(T)\lesssim \sqrt{H^3 T \log (H / \delta)} + \sum_{h,t}\beta\big\|\boldsymbol{\Sigma}_{t,h}^{-1/2}\phi_h(s_h^{(t)},a_h^{(t)})\mathbbm{1}_{\gX_{\off}}\big\|_2 + \sum_{h,t}\beta\big\|\boldsymbol{\Sigma}_{t,h}^{-1/2}\phi_h(s_h^{(t)},a_h^{(t)})\mathbbm{1}_{\gX_{\on}}\big\|_2.$$
Here, we denote 
$$\text{Reg}_{\off}(T) = \sum_{h,t}\beta\big\|\boldsymbol{\Sigma}_{t,h}^{-1/2}\phi_h(s_h^{(t)},a_h^{(t)})\mathbbm{1}_{\gX_{\off}}\big\|_2,\qquad \text{Reg}_{\on}(T) = \sum_{h,t}\beta\big\|\boldsymbol{\Sigma}_{t,h}^{-1/2}\phi_h(s_h^{(t)},a_h^{(t)})\mathbbm{1}_{\gX_{\on}}\big\|_2,$$ 
as the offline regret and online regret, respectively. 

\subsection{Offline regret control}

Now, we bound the regret on the offline partition. We first perform a similar argument to that in \cite{tan2024natural, xie2022role} to show that the sum of bonuses can be controlled by the maximum eigenvalue of the inverse weighted average covariance matrix in Lemma~\ref{lem:hybridoffline-sumbonuses}. We will then show that the maximum eigenvalue can be nicely bounded in Lemma~\ref{lem:hybridoffline-concentrability}.

\begin{lem}[Sum of Bonuses on Offline Partition]
    For any partition $\gX_{\off}, \gX_{\on}$, we can bound the sum of bonuses on the offline partition with the following:
\begin{align*}
      \text{Reg}_{\off}(T) \lesssim \sum_{h=1}^{H} \sqrt{\frac{dN_{\on}^2}{N_{\off}}\max_{\phi_h \in \Phi_{\off}} \phi_h^{\top}{\boldsymbol{\bar{\Sigma}}}_{\off,h}^{-1}\phi_h},
\end{align*}
    where $\bar{\bf{\Sigma}}_{\off,h} = {(\bf{\Sigma}}_{\off,h}+\lambda\textbf{I})/N_{\off}$ and $T= N_{\on}$.
    \label{lem:hybridoffline-sumbonuses}
\end{lem}

\begin{proof}
It is sufficient to show the following holds true
$$\sum_{t}\beta\big\|\boldsymbol{\Sigma}_{t,h}^{-1/2}\phi_h(s_h^{(t)},a_h^{(t)})\mathbbm{1}_{\gX_{\off}}\big\|_2\leq \sqrt{\frac{dN_{\on}^2}{N_{\off}}\max_{\phi_h \in \Phi_{\off}} \phi_h^{\top}{\boldsymbol{\bar{\Sigma}}}_{\off,h}^{-1}\phi_h} ,$$
then the desired inequality directly follows. With a direct calculation, one may observe that
\begin{align*}
\|\boldsymbol{\Sigma}_{t,h}^{-1/2}\phi_h(s_h^{(t)},a_h^{(t)})\mathbbm{1}_{\gX_{\off}}\|_2  &=  \sqrt{\phi_h^{\top}(s_h^{(t)}, a_h^{(t)})\boldsymbol{\Sigma}_{t,h}^{-1}\phi_h(s_h^{(t)}, a_h^{(t)})\mathbbm{1}_{\gX_{\off}}} \\
    &\lesssim \sqrt{\phi_h^{\top}(s_h^{(t)}, a_h^{(t)})(\boldsymbol{\Sigma}_{\off,h}+\lambda\mathbf{I})^{-1}\phi_h(s_h^{(t)}, a_h^{(t)}) \mathbbm{1}_{\gX_{\off}}},
\end{align*}
where the last inequality holds as $\Sigma_{\off, h} \preceq \Sigma_{t,h}$. As a result, we are able to bound the desired inequality with the maximum eigenvalue of the inverse weighted matrix,
\begin{align*}
\sum_{t }\|\boldsymbol{\Sigma}_{t,h}^{-1/2}\phi_h(s_h^{(t)},a_h^{(t)})\mathbbm{1}_{\gX_{\off}}\|_2
    &\leq N_{\on} \sqrt{\max_{\phi_h \in \Phi_{\off}} \phi_h^{\top}(\boldsymbol{\Sigma}_{\off,h}+\lambda\mathbf{I})^{-1}\phi_h} \\
    & = \sqrt{N_{\on}\frac{N_{\on}}{N_{\off}}\max_{\phi_h \in \Phi_{\off}} \phi_h^{\top}{\boldsymbol{\bar{\Sigma}}}_{\off,h}^{-1}\phi_h},
\end{align*}
where $\bar{\bf{\Sigma}}_{\off,h} = {(\bf{\Sigma}}_{\off,h}+\lambda\textbf{I})/N_{\off}$. As $\beta = \tilde{O}(\sqrt{d})$, we obtain the bound we desired:
\begin{align*}
    &\sum_{t}\beta\|\boldsymbol{\Sigma}_{t,h}^{-1/2}\phi_h(s_h^{(t)},a_h^{(t)})\mathbbm{1}_{\gX_{\off}}\|_2 \leq \sqrt{dN_{\on}\frac{N_{\on}}{N_{\off}}\max_{\phi_h \in \Phi_{\off}} \phi_h^{\top}{\boldsymbol{\bar{\Sigma}}}_{\off,h}^{-1}\phi_h}.    
    \end{align*}
\end{proof}

\begin{lem}[Partial Concentrability Bound]
    For any partition $\gX_{\off}, \gX_{\on}$, we have that $$\sum_{h=1}^{H}\max_{\phi_h \in \Phi_{\off}} \sqrt{\phi_h^{\top}{\boldsymbol{\bar{\Sigma}}}_{\off,h}^{-1}\phi_h }\lesssim {\sqrt{c_{\off}(\gX_{\off})^2 H^3}},$$
    when $N_{\on}, N_{\off} \geq \tilde{\Omega}(d^{13}H^{14}),$ where we define $\bar{\bf{\Sigma}}_{\off,h} = {(\bf{\Sigma}}_{\off,h}+\lambda\textbf{I})/N_{\off}.$
    \label{lem:hybridoffline-concentrability}
\end{lem}

\begin{proof}
Similar to the definition of $\bar{\boldsymbol{\Sigma}}_{\off,h}$, we define $\bar{\boldsymbol{\Lambda}}_{\off,h} = {(\boldsymbol{\Lambda}}_{\off,h}+\lambda\textbf{I})/N_{\off}$ in a similar way. Then, one may observe that
\begin{align*}
    \max_{\phi_h \in \Phi_{\off}} \left(\phi_h^{\top}\bar{\bf{\Lambda}}_{\off,h}^{-1}
    \phi_h\right)
    &=\max_{\phi_h \in \Phi_{\off}} \left(\phi_h^{\top}\left(\frac{1}{N_{\off}}\left(\sum_{n=1}^{N_{\off}} \phi_{n,h}\phi_{n,h}^{\top}+\lambda\bf{I}\right)\right)^{-1}\phi_h\right) \\
    &\leq \max_{\phi_h \in \Phi_{\off}} \phi_h^{\top}\E_{\mu_h}[\bar{\bf{\Lambda}}_{\off,h}]^{-1}\phi_h \left(1-\sqrt{\frac{2}{N_{\text {off }}} \log \left(\frac{4d}{\delta}\right)}\right)^{-1},
\end{align*}
where the last line holds by an application of the Matrix Chernoff inequality. Then, we may further bound the quantity with the partial offline all-policy concentrability coefficient,

\begin{align*}
\max_{\phi_h \in \Phi_{\off}} \left(\phi_h^{\top}\bar{\bf{\Lambda}}_{\off,h}^{-1}
    \phi_h\right)
    &\lesssim \inf_{\gX_{\off}, \gX_{\on}} \max_h\frac{1}{\lambda_{d_{\off}}\left(\E_{\mu}(\gP_{\off}\phi_{h})(\gP_{\off}\phi_h)^{\top}\right)} \\
    &= \inf_{\gX_{\off}, \gX_{\on}} \max_h\frac{1}{\lambda_{\min}\left(\E_{\mu}(U_{\off}^{\top}\phi_h)(U_{\off}^{\top}\phi_h)^{\top}\right)} \\
    &= c_{\off}(\gX_{\off}).
\end{align*}

To tighten the dependence of the regret of the offline partition on $H$, we again employ a truncation argument that used in Lemma~\ref{lem:hybridoffline-H3}. Recall that in Section B of the appendix in \cite{he2023nearly}, by the total variance lemma of \cite{jin2019provably}, it holds that $$\sum_{t=1}^{T} \sum_{h=1}^H \sigma_{t, h}^2 \leq \widetilde{O}\left(H^2 T+d^{10.5} H^{16}\right).$$

Again, recall that we have
\begin{align*}
    &\sum_{h,t }\|\boldsymbol{\Sigma}_{t,h}^{-1/2}\phi_h(s_h^{(t)},a_h^{(t)})\mathbbm{1}_{\gX_{\off}}\|_2  \\
    &\lesssim \sqrt{H^2N_{\on}\frac{N_{\on}}{N_{\off}}\max_{\phi_h \in \Phi_{\off}} \left(\phi_h^{\top}\left(\frac{1}{N_{\off}}\left(\sum_{n=1}^{N_{\off}} \bar{\sigma}_{n,h}^{-2}\phi_{n,h}\phi_{n,h}^{\top}+\lambda\bf{I}\right)\right)^{-1}\phi_h\right)}.
\end{align*}

As $\bar{\sigma}_{n,h}^2 = \max\left\{\sigma_{n,h}^2, H, 4d^6H^4||\phi_{n,h}||_{\Sigma_{n,h}^{-1}}\right\}$. Consider the sets 
\begin{align*}
    \gI_{1} &= \big\{n \in [N_{\off}] : \forall h:\; \bar{\sigma}_{n,h}^2 = \max(\sigma_{n,h}^2,H)\big\},\qquad \gI_{2} = \gI_1^c.
\end{align*}

Here, $\gI_2$ roughly correspond to the ``bad'' set of trajectories where there exists some timestep $h$ such that $\bar{\sigma}_{n,h}^2 > \max\{\sigma_{n,h}^2,H\},$ and $\gI_1$ to be the ``good'' set of trajectories where the monotonic variance estimator is controlled.

We need to bound the cardinality of the latter before employing our truncation argument on the estimated variances. As we note that for all $n \in \gI_{2}$ we have that $\max_{h\in[H]}\sqrt{\phi_{n,h}^{\top}\Sigma_{n,h}^{-1}\phi_{n,h}}\geq 1/(4d^6H^2)$, which indicates that
\begin{align*}
  \sum_{h=1}^{H} \min\big\{1,16d^{12}H^4\phi_{n,h}^{\top}\Sigma_{n,h}^{-1}\phi_{n,h}\big\}\geq 1,
\end{align*}
and so we can conclude that
\begin{align*}
    |\gI_{2}| 
    \leq \sum_{h=1}^{H}\sum_{n=1}^{N_{\off}} \min\big\{1,16d^{12}H^4\phi_{n,h}^{\top}\Sigma_{n,h}^{-1}\phi_{n,h}\big\} \lesssim d^{13}H^5\log(1+N/d),
\end{align*}
by Lemma D.5 of \cite{zhou2022computationally} and the fact that $||\phi_{n,h}/\bar{\sigma}_{n,h}||^2 \leq 1/H^2$. 
As we require in Theorem \ref{thm:hybridonline} that $N_{\on}, N_{\off} = \tilde{\Omega}(d^{13}H^{14})$, we come to the following result
$$|\gI_{2}|/N_{\off} \lesssim 8d^{13}H^5\log(1+N/d)/N_{\off} = \tilde{o}(1), \quad|\gI_1|/N_{\off} = 1-\tilde{o}(1).$$

Informally, this means that the proportion of trajectories in the ``bad set'' $\gI_{2}$ is asymptotically zero, and the proportion in the ``good set'' $\gI_{1}$ is asymptotically one. As for every $n \in \gI_{1}$ we have that for any $h\in[H]$,
\begin{align*}
    &\max_{\phi_h \in \Phi_{\off}}\left(\phi_h^{\top} \bar{\boldsymbol{\Sigma}}_{\off, h}^{-1} \phi_h\right)\\
    &=\max_{\phi_h \in \Phi_{\off}} \left(\phi_h^{\top}\left(\frac{1}{N_{\off}}\left(\sum_{n=1}^{N_{\off}} \bar{\sigma}_{n,h}^{-2}\phi_{n,h}\phi_{n,h}^{\top}+\lambda\bf{I}\right)\right)^{-1}\phi_h\right) \\
    &=\max_{\phi_h \in \Phi_{\off}} N_{\off}\left(\phi_h^{\top}\left(\sum_{n=1}^{N_{\off}} \bar{\sigma}_{n,h}^{-2}\phi_{n,h}\phi_{n,h}^{\top}+\lambda\bf{I}\right)^{-1}\phi_h\right) \\
    &\leq \max_{\phi_h \in \Phi_{\off}} N_{\off}\left(\phi_h^{\top}\left(\sum_{n\in \gI_{1}} \frac{\phi_{n,h}\phi_{n,h}^{\top}}{\sigma_{n,h}^2 + H}+\lambda\bf{I}\right)^{-1}\phi_h\right).
\end{align*}


Now we invoke the total variance lemma. Recall that in Section B of the appendix in \cite{he2023nearly}, by the total variance lemma of \cite{jin2019provably}, if $N_{\off} \geq  \tilde{\Omega}(d^{10.5}H^{14})$, it holds that $$\frac{1}{N_{\off}}\sum_{n=1}^{N_{\off}} \sum_{h=1}^H \sigma_{n, h}^2 = \widetilde{O}\left(H^2+d^{10.5} H^{16}/N_{\off}\right) =  \widetilde{O}\left(H^2\right).$$

With a direct application of Lemma~\ref{lem:total-variance-concentration}, as we set $T=\tilde{O}(H)$ and $\gamma = c_{\off}(\gX_{\off})/N_{\off}$, we will then get to

$$\sum_{h=1}^{H}\max_{\phi_h \in \Phi_{\off}} \sqrt{\phi_h^{\top}{\boldsymbol{\Sigma}}_{\off,h}^{-1}\phi_h }\lesssim \frac{c_{\off}(\gX_{\off})H}{N_{\off}}\sqrt{N_{\off}H} = \sqrt{\frac{c_{\off}(\gX_{\off})H^3}{N_{\off}}},
$$
which indicates that 
$$\sum_{h=1}^{H}\max_{\phi_h \in \Phi_{\off}} \sqrt{\phi_h^{\top}{\boldsymbol{\bar{\Sigma}}}_{\off,h}^{-1}\phi_h }\lesssim\sqrt{c_{\off}(\gX_{\off})^2 H^3}. $$
 \end{proof}



Now, from Lemmas \ref{lem:hybridoffline-sumbonuses} and \ref{lem:hybridoffline-concentrability}, for any partition 
$\gX_{\off}, \gX_{\on}$, the offline regret satisfies
\begin{align*}
    \text{Reg}_{\off}(T)  
    \lesssim \sum_{h=1}^{H}\sqrt{dN_{\on}\frac{N_{\on}}{N_{\off}}\max_{\phi_h\in \Phi_{\off}} \phi_h^{\top}{\boldsymbol{\bar{\Sigma}}}_{\off,h}^{-1}\phi_h }\lesssim \sqrt{c_{\off}(\gX_{\off})^2dH^3N_{\on}\frac{N_{\on}}{N_{\off}}}.
\end{align*}
\subsection{Online regret control}


We will then bound the online term, $\text{Reg}_{\on}(T)$. \cite{he2023nearly} show in Lemma E.1 that it is possible to use Cauchy-Schwarz to bound this by
$$\text{Reg}_{\on}(T) =  \widetilde{O}\left(d^4 H^8+\beta d^7 H^5+\beta \sqrt{d H T+d H \sum_{t=1}^{T} \sum_{h=1}^H \sigma_{t, h}^2}\right),$$
and in Section B of the appendix, state that by the total variance lemma of \cite{jin2019provably}, $$\sum_{t=1}^{T} \sum_{h=1}^H \sigma_{t, h}^2 \leq \widetilde{O}\left(H^2 T+d^{10.5} H^{16}\right)$$

We will seek to use the online partition to tighten the dimensional dependence in the first result accordingly.

\begin{lem}[Modified Lemma E.1 in \cite{he2023nearly}]
\label{lem:hybridonline-sumbonuses-xon}
    For any parameters $\beta^{\prime} \geq 1$ and $C \geq 1$, and any partition $\gX_{\off}, \gX_{\on}$, the summation of bonuses on the online partition is upper bounded by
    \begin{align*}
        &\sum_{t=1}^{T} \min \left(\beta^{\prime} \sqrt{\phi\left(s_h^{(t)}, a_h^{(t)}\right)^{\top} \boldsymbol{\Sigma}_{t, h}^{-1} \phi\left(s_h^{(t)}, a_h^{(t)}\right)}\mathbbm{1}_{\gX_{\on}}, C\right) \\
        &\leq 4 d^4 H^6 C \iota+10 \beta^{\prime} d_{\on}^5 H^4 \iota+2 \beta^{\prime} \sqrt{2 d_{\on} \iota \sum_{t=1}^{T}\left(\sigma_{t, h}^2+H\right)}
    \end{align*}
where $\iota=\log (1+N /(d \lambda))$.
\end{lem}
\begin{proof}
    For each horizon $h\in[H]$, we first note that the summation can be bounded by the sum of two terms, where the first term is tight-bounded and the second term stands for a tail event where $\phi^{T}\Sigma^{-1}\phi$ gets large.
    \begin{align*}
& \sum_{t=1}^{T} \min \left(\beta^{\prime} \sqrt{{\phi}\left(s_h^{(t)}, a_h^{(t)}\right)^{\top} \boldsymbol{\Sigma}_{t, h}^{-1} {\phi}\left(s_h^{(t)}, a_h^{(t)}\right)}\mathbbm{1}_{\gX_{\on}}, C\right) \\
& \leq \sum_{t=1}^{T} \beta^{\prime} \min \left(\sqrt{\phi\left(s_h^{(t)}, a_h^{(t)}\right)^{\top} \boldsymbol{\Sigma}_{t, h}^{-1} {\phi}\left(s_h^{(t)}, a_h^{(t)}\right)}\mathbbm{1}_{\gX_{\on}}, 1\right)\\
&\qquad + C \sum_{t=1}^{T} \mathbbm{1}\left\{\sqrt{{\phi}\left(s_h^{(t)}, a_h^{(t)}\right)^{\top} \boldsymbol{\Sigma}_{t, h}^{-1} {\phi}\left(s_h^{(t)}, a_h^{(t)}\right)}\mathbbm{1}_{\gX_{\on}} \geq 1\right\}.
\end{align*}

We first bound $\sum_{t=1}^{T} \beta^{\prime} \min \left(\sqrt{\phi\left(s_h^{(t)}, a_h^{(t)}\right)^{\top} \boldsymbol{\Sigma}_{t, h}^{-1} {\phi}\left(s_h^{(t)}, a_h^{(t)}\right)}\mathbbm{1}_{\gX_{\on}}, 1\right)$, using a variant of Lemma B.1 from \cite{zhou2022computationally} in Lemma \ref{lem:zhou2022-B1-modified}. With this, we have that
\begin{align*}
    &\sum_{t=1}^{T} \beta^{\prime} \min \left(\sqrt{\phi\left(s_h^{(t)}, a_h^{(t)}\right)^{\top} \boldsymbol{\Sigma}_{t, h}^{-1} {\phi}\left(s_h^{(t)}, a_h^{(t)}\right)}\mathbbm{1}_{\gX_{\on}}, 1\right)\\
    &\leq \sum_{t=1}^{T} \beta^{\prime} \min \left(\sqrt{\phi\left(s_h^{(t)}, a_h^{(t)}\right)^{\top} \left(\sum_{n=1}^{N_{\off}+t}(\phi_{n,h}\mathbbm{1}_{\gX_{\on}})(\phi_{n,h}\mathbbm{1}_{\gX_{\on}})^{\top}+\lambda\mathbf{I}_d\right)^{-1} {\phi}\left(s_h^{(t)}, a_h^{(t)}\right)}\mathbbm{1}_{\gX_{\on}}, 1\right)\\
    &\leq 10 \beta^{\prime} d_{\on}^5 H^4 \iota+2 \beta^{\prime} \sqrt{2 d_{\on} \iota \sum_{k=1}^K\left(\sigma_{k, h}^2+H\right)},
\end{align*}
where $\iota = \log(1+N/(d\lambda))$. 

From this, it suffices to follow the rest of the proof of Lemma E.1 from \cite{he2023nearly} to bound the remaining term by 
\begin{align*}
    \sum_{t=1}^{T} \mathbbm{1}\left\{\sqrt{\boldsymbol{\phi}\left(s_h^{(t)}, a_h^{(t)}\right)^{\top} \boldsymbol{\Sigma}_{t, h}^{-1} \boldsymbol{\phi}\left(s_h^{(t)}, a_h^{(t)}\right)}\mathbbm{1}_{\gX_{\on}} \geq 1\right\} \leq 4d^4H^6C\iota.
\end{align*}
\end{proof}

As a result, we obtain the following bound for the online regret
$$\text{Reg}_{\on}(T)\lesssim d^7 H^9 + \beta\sqrt{d_{\on}d H^3 T}.$$

\subsection{Putting everything together}
Combining our results in E.4 and E.5, we come to the bound of total regret  that
\begin{align*}
  \text{Reg}(N_{\on})\lesssim \sqrt{H^3N_{\on}\log(H/\delta)} + \sqrt{c_{\off}(\gX_{\off})^2dH^3N_{\on}\frac{N_{\on}}{N_{\off}}} +  \sqrt{d_{\on}d H^3N_{\on}} + d^7 H^9.
\end{align*}
When we set $N_{\on}, N_{\off} = \tilde{\Omega}(d^{13}H^{14})$ and choose $\gX_{\off}$, $\gX_{\on}$ be the partition that minimize the right hand side, we have
\begin{align*}
     \text{Reg}(N_{\on})\lesssim \inf_{\gX_{\off}, \gX_{\on}} \left(\sqrt{c_{\off}(\gX_{\off})^2dH^3N_{\on}\frac{N_{\on}}{N_{\off}}} + \sqrt{d_{\on}d H^3N_{\on}}\right),
\end{align*}
proving Theorem \ref{thm:hybridonline}.



\section{OPTCOV from \cite{wagenmaker2023instancedependent}}

We lean on the OPTCOV algorithm from \cite{wagenmaker2023leveraging} for reward-agnostic exploration , first proposed in \cite{wagenmaker2023instancedependent}, as well as the Frank-Wolfe subroutine used, for completeness. 

\begin{algorithm}[H]
    \caption{Collection of Optimal Covariates (OPTCOV), \cite{wagenmaker2023leveraging}}
    \begin{algorithmic}[1]
        \State \textbf{Input:} functions to optimize $(f_i)_i$, constraint tolerance $\epsilon$, confidence $\delta$.
        \For{$i=1,2,3,...$}
            \State Set the number of iterates $T_i \leftarrow 2^i$, episodes per iterate $K_i \leftarrow 2^i$.
            \State Play any policy for $K_i$ episodes to collect covariates $\boldsymbol{\Gamma}_0$ and data $\mathfrak{D}_0$.
            \State Initialize covariance matrix $\boldsymbol{\Lambda}_1 \gets \boldsymbol{\Gamma}_0/K$.
            \For{$t=1,...,T_i$}
                \State Run FORCE \citep{wagenmaker2022firstorder} or another regret-minimizing algorithm on the exploration-focused synthetic reward $g_h^{(t)}(s,a) \propto \text{tr}(-\nabla_{\boldsymbol{\Lambda}} f_i(\boldsymbol{\Lambda})|_{\boldsymbol{\Lambda}=\boldsymbol{\Lambda}_t \phi(s,a)\phi(s,a)^{\top}})$.
                \State Collect covariates $\boldsymbol{\Gamma}_t$, data $\mathfrak{D}_t$. 
                \State Perform Frank-Wolfe update: $\boldsymbol{\Gamma}_{t+1} \gets (1-\frac{1}{t+1})\boldsymbol{\Lambda}_{t} + \frac{1}{t+1}\boldsymbol{\Gamma}_{t}/{K_i}$.
            \EndFor
            \State Assign $\widehat{\boldsymbol{\Lambda}_i} \gets \boldsymbol{\Lambda}_{T_i+1}, \mathfrak{D}_i \gets \cup_{t=0}^{T_i}\mathfrak{D}_t$.
            \If{$f_i(\widehat{\boldsymbol{\Lambda}_i})\leq K_iT_i\epsilon$}
                \State{\bfseries Return:} $\widehat{\Lambda}, K_i T_i, \mathfrak{D}_i$.
            \EndIf
        \EndFor
    \end{algorithmic}
    \label{alg:optcov}
\end{algorithm}

The algorithm essentially performs the doubling trick to determine how many samples to collect, terminating when the minimum eigenvalue of the covariance matrix is above the set tolerance. 

\cite{wagenmaker2023leveraging} then prove the following guarantee for OPTCOV in the hybrid setting:
\begin{lem}[Termination of OPTCOV, Lemma C.2 \citep{wagenmaker2023leveraging}]
\label{lem:optcov-lemma}
Let
$$
f_i(\boldsymbol{\Lambda})=\frac{1}{\eta_i} \log \left(\sum_{\phi \in \Phi} e^{\eta_i\|\boldsymbol{\phi}\|_{\mathbf{A}_i(\boldsymbol{\Lambda})^{-1}}^2}\right), \quad \mathbf{A}_i(\boldsymbol{\Lambda})=\boldsymbol{\Lambda}+\frac{1}{T_i K_i} \boldsymbol{\Lambda}_{0, i}+\frac{1}{T_i K_i} \boldsymbol{\Lambda}_{\off }
$$
for some $\boldsymbol{\Lambda}_{0, i}$ satisfying $\boldsymbol{\Lambda}_{0, i} \succeq \boldsymbol{\Lambda}_0$ for all $i$, and $\eta_i=2^{2 i / 5}$. Let $\left(\beta_i, M_i\right)$ denote the smoothness and magnitude constants for $f_i$. Let $(\beta, M)$ be some values such that $\beta_i \leq \eta_i \beta, M_i \leq M$ for all $i$.
Then, if we run OPTCOV on $\left(f_i\right)_i$ with constraint tolerance $\epsilon$ and confidence $\delta$, we have that with probability at least $1-\delta$, it will run for at most
$$
\begin{gathered}
\max \left\{\min _N 16 \boldsymbol{N} \quad \text { s.t. } \quad \inf _{\boldsymbol{\Lambda} \in \boldsymbol{\Omega}} \max _{\phi \in \Phi} \phi^{\top}\left(N \boldsymbol{\Lambda}+\boldsymbol{\Lambda}_0+\boldsymbol{\Lambda}_{\off }\right)^{-1} \phi \leq \frac{\epsilon}{6},\right. \\
\left.\frac{\operatorname{poly}(\beta, d, H, M, \log 1 / \delta)}{\epsilon^{4 / 5}}\right\} .
\end{gathered}
$$
episodes, and will return data $\left\{\phi_\tau\right\}_{\tau=1}^N$ with covariance $\widehat{\boldsymbol{\Sigma}}_N=\sum_{\tau=1}^N \phi_\tau \phi_\tau^{\top}$ such that
$$
f_{\hat{i}}\left(N^{-1} \widehat{\boldsymbol{\Sigma}}_N\right) \leq N \epsilon
$$
where $\widehat{i}$ is the iteration on which OPTCOV terminates.
\end{lem}

We use this to obtain a modified guarantee for OPTCOV that does not require a call to the CONDITIONEDCOV algorithm of \cite{wagenmaker2023instancedependent}.

\begin{lem}[Modified Bound on OPTCOV, Theorem 4, \cite{wagenmaker2023leveraging}]
\label{lem:optcov-correct}
Consider running OPTCOV with some $\epsilon_{\exp }>0$ and functions $f_i$ as defined in Lemma \ref{lem:optcov-lemma}, instantiated with the regularization $\bar{\lambda} \geq 0$.
Then with probability $1-\delta$, this procedure will collect at most
$$
\max \left\{\min _N C \cdot N \text { s.t. } \inf _{\boldsymbol{\Lambda} \in \boldsymbol{\Omega}} \max _{\phi \in \Phi} \boldsymbol{\phi}^{\top}\left(N(\boldsymbol{\Lambda}+\bar{\lambda} I)+\boldsymbol{\Lambda}_{\mathrm{off}}\right)^{-1} \boldsymbol{\phi} \leq \frac{\epsilon_{\mathrm{exp}}}{6},\frac{\operatorname{poly}\left(d, H, c_{\on}(\gX_{\on}),\log 1 / \delta\right)}{\epsilon_{\exp }^{4 / 5}}\right\} 
$$
episodes, and will produce covariates $\widehat{\boldsymbol{\Sigma}}$ such that
$$
\max _{\boldsymbol{\phi}_h \in \Phi}\phi_h\left(\widehat{\boldsymbol{\Sigma}}+\bar{\lambda}\textbf{I}+\boldsymbol{\Lambda}_{\mathrm{off}}\right)^{-1}\phi_h \leq \epsilon_{\exp }.
$$
\end{lem}
\begin{proof}
This is essentially the proof of Theorem 4 in \cite{wagenmaker2023leveraging}, except where we chase around a few terms that differ in the analysis. By Lemma D.5 of \cite{wagenmaker2023instancedependent}, it suffices to bound the smoothness constants of $f_i(\boldsymbol{\Lambda})$ by
$$
L_i=\frac{1}{\bar{\lambda}^2}, \quad \beta_i= \frac{2}{\bar{\lambda}^3} \left(1+\frac{\eta_i}{\bar{\lambda}}\right), \quad M_i=\frac{1}{\bar{\lambda}^2}.
$$    

Assume that the termination condition of OPTCOV is met for $\widehat{i}$ satisfying
$$
\widehat{i} \leq \log \left(\operatorname{poly}\left(\frac{1}{\epsilon_{\exp }}, d, H, \log 1 / \delta, c_{\on}(\gX_{\on}), \bar{\lambda}\right)\right) .
$$

We assume this holds and justify it at the conclusion of the proof. For notational convenience, define
$$
\iota:=\operatorname{poly}\left(\log \frac{1}{\epsilon_{\exp }}, d, H, \log 1 / \delta, c_{\on}(\gX_{\on}), \bar{\lambda}\right) .
$$

Given this upper bound on $\widehat{i}$, set
$$
L=M:=\frac{1}{\bar{\lambda}^2}, \quad \beta:=\iota .
$$

With this choice of $L, M, \beta$, we have $L_i \leq L, M_i \leq M, \beta_i \leq \eta_i \beta$ for all $i \leq \widehat{i}$.

Now apply Lemma \ref{lem:optcov-lemma} with $\boldsymbol{\Lambda}_0=\bar{\lambda} \cdot \textbf{I}$ and get that, with probability at least $1-\delta$, OPTCOV terminates after at most
$$
\begin{gathered}
\max \left\{\min _N 16 N \quad \text { s.t. } \quad \inf _{\boldsymbol{\Lambda} \in \boldsymbol{\Omega}} \max _{\boldsymbol{\phi} \in \Phi} \boldsymbol{\phi}^{\top}\left(N \boldsymbol{\Lambda}+\bar{\lambda} \cdot I+\boldsymbol{\Lambda}_{\mathrm{off}}\right)^{-1} \boldsymbol{\phi} \leq \frac{\epsilon_{\exp }}{6}\right. \\
\left.\frac{\operatorname{poly}\left(d, H, \underline{\lambda}, c_{\off}(\gX_{\off}), \log 1 / \epsilon_{\exp }, \log 1 / \delta\right)}{\epsilon_{\exp }^{4 / 5}}\right\}
\end{gathered}
$$
episodes, and returns data $\left\{\boldsymbol{\phi}_\tau\right\}_{\tau=1}^N$ with covariance $\widehat{\boldsymbol{\Sigma}}=\sum_{\tau=1}^N \boldsymbol{\phi}_\tau \boldsymbol{\phi}_\tau^{\top}$ such that
$$
f_{\hat{i}}\left(N^{-1} \widehat{\boldsymbol{\Sigma}}\right) \leq N \epsilon_{\mathrm{exp}}
$$
where $\widehat{i}$ is the iteration on which OPTCOV terminates.

By Lemma D.1 of \cite{wagenmaker2023instancedependent} we have
$$
N \cdot \max _{\phi_h \in \Phi} \phi_h \left(\widehat{\boldsymbol{\Sigma}}+\boldsymbol{\Lambda}_{\hat{i}, 0}+\boldsymbol{\Lambda}_{\mathrm{off}}\right)^{-1} \phi_h \leq f_{\hat{i}}\left(N^{-1} \widehat{\boldsymbol{\Sigma}}\right),
$$
and the upper bound on the tolerance follows from Lemma D.8 of \cite{wagenmaker2023instancedependent}.

It remains to justify the bound on $\widehat{i}$. We do so with the same argument that \cite{wagenmaker2023leveraging} use. Note that by the definition of OPTCOV, if we run for a total of $\bar{N}$ episodes, we can bound $\widehat{i} \leq \frac{1}{4} \log _2(\bar{N})$. However, we see that the bound on $\widehat{i}$ given above upper bounds $\frac{1}{4} \log _2(\bar{N})$ for $\bar{N}$ the upper bound on the number of samples collected by OPTCOV stated above. Thus, the bound on $\hat{i}$ is valid.
\end{proof}

\section{Miscellanous lemmas}
\label{app:misc}
\begin{lem}
    \label{lem:total-variance-concentration}
    Let $\Phi\subset\mathbb{R}^{d}$ be a linear subspace. Suppose $\{\phi_{h,n}\}_{h\in[H],n\in[N]}\in\Phi$ be a collection of unit vectors and $\{\sigma_{h,n}\}_{h\in[H],n\in[N]}\in\mathbb{R}_{+}$ be a collection of positive real numbers with mean
    $\bar{\sigma}=(NH)^{-1}\sum_{h,n}\sigma_{h,n}$.
Suppose it holds that $\max_{h\in[H]}\max_{\phi_h\in\Phi}(\phi_h^{T}\Lambda_h^{-1}\phi_h)\leq \gamma,$
then the following result satisfies
$$ \sum_{h=1}^{H}\max_{\phi_h\in\Phi}\sqrt{\phi_h^{T}\Sigma_h^{-1}\phi_h}\lesssim \gamma H\sqrt{N\bar{\sigma}},$$
with 
$$ \Lambda_h = \sum_{n=1}^{N}\phi_{h,n}\phi_{h,n}^T + \lambda I_d,\qquad \Sigma_h = \sum_{n=1}^{N}\frac{\phi_{h,n}\phi_{h,n}^T}{\sigma_{h,n}} + \lambda I_d. $$
\end{lem}
\begin{proof}
      First, we denote
      $\bar{\sigma}_h = N^{-1}\sum_{n}\sigma_{h,n}.$
      Informally, this implies that most individuals of $\sigma_{h,\cdot}$ is asymptotically on the order of $\bar{\sigma}_h $, with only a small amount of individuals being higher in order. To rule out the effect of the ``large'' ones, we group them into the following collection of sets:
    $$\mathcal{E}_h(C_h) = \{n\in[N]: \sigma_{h,n}\geq C_h\bar{\sigma}_h  \}.$$
    Here, we leave the choice of the truncation level $C_h$ open for now, but note that we allow the truncation levels $C_h$ vary across different timesteps $h$ and related to $\bar{\sigma}_h $.
    It follows by definition that $\sum_{h=1}^{H} \bar{\sigma}_h  = H\bar{\sigma}.$
    From an application of Markov's Inequality, the cardinality of set $\mathcal{E}_h(C_h)$ can be upper bounded as
    $$|\mathcal{E}_h(C_h)|\leq \frac{N }{C_h}.$$

    
    We now choose the truncation level $C_h$. To do so, we follow the steps below to quantify the effect induced by the trajectories with high variance (i.e. those that belong to $\mathcal{E}_h(C_h)$):
    \begin{align}
      \min_{\phi_h\in\Phi}\phi_h^{\top}{\Sigma_h^{\star}}\phi_h & \geq \min_{\phi_h\in\Phi}\phi_h^{\top}\left(\sum_{n=1}^{N}\frac{\phi_{h,n}\phi_{h,n}^{T}}{\sigma_{h,n}}\right)\phi_h\nonumber\\
      &\geq \min_{\phi_h\in\Phi}\phi_h^{\top}\bigg(\sum_{n\in[N]\backslash \mathcal{E}_h(C_h)}\frac{\phi_{h,n}\phi_{h,n}^{T}}{\sigma_{h,n} }\bigg)\phi_h\nonumber\\
      &\geq \frac{1}{C_h \bar{\sigma}_h}\min_{\phi_h\in\Phi}\phi_h^{\top}\bigg(\sum_{n\in[N]\backslash \mathcal{E}_h(C_h)}\phi_{h,n}\phi_{h,n}^{T}\bigg)\phi_h.\nonumber
    \end{align}
     We now utilize a basic matrix inequality that for any matrix $A,B$, we have $$\min_{\phi_h\in\Phi}\phi_h^{\top}A\phi_h\geq \min_{\phi_h\in\Phi}\phi_h^{\top}(A+B)\phi_h-\max_{\phi_h\in\Phi}\phi_h^{\top}B\phi_h,$$ 
     which allows us to further bound $\min_{\phi_h\in\Phi}\phi_h^{\top}{\Sigma_h^{\star}}\phi_h$ as
     \begin{align}
       \min_{\phi_h\in\Phi}\phi_h^{\top}{\Sigma_h^{\star}}\phi_h & \geq \frac{1}{C_h \bar{\sigma}_h}\min_{\phi_h\in\Phi}\phi_h^{\top}\bigg(\sum_{n=1}^{N}\phi_{h,n}\phi_{h,n}^{T} + \lambda I_d\bigg)\phi_h \nonumber\\ 
       &\qquad - \frac{1}{C_h \bar{\sigma}_h}\max_{\phi_h\in\Phi}\phi_h^{\top}\bigg(\sum_{n\in\mathcal{E}_h(C_h)}\phi_{h,n}\phi_{h,n}^{T} + \lambda I_d\bigg)\phi_h \nonumber\\
       &\gtrsim \frac{1}{C_h \bar{\sigma}_h}\bigg(\gamma^{-1} - \frac{N}{C_h}-\lambda\bigg),\nonumber
     \end{align}
    This leads to the following result:
    $$\min_{\phi_h\in\Phi}\phi_h^{\top} \Lambda_h \phi_h = \min_{\phi_h\in\Phi}(\phi_h^{\top} \Lambda_h^{-1} \phi_h)^{-1}  \gtrsim \left\{\max\left(\frac{c_{\off}(\gX_{\off})}{N_{\off}},\frac{d_{\on}}{N_{\on}}\right)\right\}^{-1} = \gamma^{-1}, $$
    where the first equality holds because $\Lambda_h$ is a linear transformation on the subspace $\Phi$. Equivalently, this holds from the variational characterization of the eigenvalues and the fact that the largest absolute eigenvalue is equal to the inverse of the smallest absolute eigenvalue of the inverse. As a result, in order to rule out the effect of the ``high variance trajectories",  we select the truncation level $\delta_h$ such that 
    $N/C_h = \Theta(\gamma^{-1}),$
     implying $C_h = \Theta(N\gamma)$. Hence, we obtain the following lower bound: $$ \min_{\phi_h\in\Phi}\phi_h^{\top}{\Sigma_h^{\star}}\phi_h\gtrsim \frac{1}{\gamma^2N \bar{\sigma}_h}.$$
     
     Finally, we note that
     \begin{align}
     \sum_{h=1}^{H}\max_{\phi_h\in\Phi}\sqrt{\phi_h^{\top}{\Sigma_h^{\star}}^{-1}\phi_h}&=\sum_{h=1}^{H}\bigg(\min_{\phi_h\in\Phi}\sqrt{\phi_h^{\top}\Sigma_h^{\star}\phi_h}\bigg)^{-1}\lesssim \gamma\sqrt{N}\sum_{h=1}^{H}\sqrt{\bar{\sigma}_h}\nonumber\leq\gamma H\sqrt{N\bar{\sigma}}.
     \end{align}
\end{proof}


\begin{lem}[Modified Lemma B.1 from \cite{zhou2022computationally}]
    \label{lem:zhou2022-B1-modified}
    Let $\gX_{\off}, \gX_{\on}$ be a partition of $\gS \times\gA \times[H]$, such that their images under the feature map, $\Phi_{\off}, \Phi_{\on}$ are subspaces of dimension $d_{\off}, d_{\on}$ respectively. Let $\left\{\sigma_k, \beta_k\right\}_{k \geq 1}$ be a sequence of non-negative numbers, $\alpha, \gamma>0,\left\{\mathbf{x}_k\right\}_{k \geq 1} \subset \mathbb{R}^d$ and $\left\|\mathbf{x}_k\right\|_2 \leq L$. Let $\left\{\mathbf{Z}_k\right\}_{k \geq 1}$ and $\left\{\bar{\sigma}_k\right\}_{k \geq 1}$ be recursively defined as follows: $\mathbf{Z}_1=\lambda \mathbf{I} + \mathbf{Z}_{\off}$ for some symmetric matrix $\mathbf{Z}_{\off}$, where $N=N_{\off}+K$, and we have
$$
\forall k \geq 1, \bar{\sigma}_k=\max \left\{\sigma_k, \alpha, \gamma\left\|\mathbf{x}_k\right\|_{\mathbf{z}_k^{-1}}^{1 / 2}\right\}, \mathbf{Z}_{k+1}=\mathbf{Z}_k+\mathbbm{1}_{\gX_{\on}}\mathbf{x}_k \mathbf{x}_k^{\top} / \bar{\sigma}_k^2
$$

Let $\iota=\log \left(1+N L^2 /\left(d \lambda \alpha^2\right)\right)$. Then we have
$$
\sum_{k=1}^K \min \left\{1, \beta_k\left\|\mathbf{x}_k\right\|_{\mathbf{z}_k^{-1}}\mathbbm{1}_{\gX_{\on}}\right\} \leq 2 d_{\on} \iota+2 \max _{k \in[K]} \beta_k \gamma^2 d_{\on} \iota+2 \sqrt{d_{\on} \iota} \sqrt{\sum_{k=1}^K \beta_k^2\left(\sigma_k^2+\alpha^2\right)}.
$$
\end{lem}

\begin{proof}
    The proof roughly follows that of Lemma B.1 in \cite{zhou2022computationally}, except that we have to make modifications as necessary to tighten the dimension dependence to $d_{\on}$ and incorporate the offline data. 
    
    Decompose the set $[K]$ into a union of two disjoint subsets $[K]=\mathcal{I}_1 \cup \mathcal{I}_2$,
    $$\mathcal{I}_1=\left\{k \in[K]:\left\|\mathbf{x}_k / \bar{\sigma}_k\right\|_{\mathbf{Z}_k^{-1}}\mathbbm{1}_{\gX_{\on}} \geq 1\right\}, \mathcal{I}_2=[K] \backslash \mathcal{I}_1 .$$

    Then the following upper bound of $|\gI_{1}|$ holds, where the projector $\gP_{\on}$ onto $\Phi_{\on}$ has the decomposition $\gP_{\on}= U_{\on}U_{\on}^{\top}$ by the thin SVD, and we write $\mathbf{u}_k = U_{\on}^{\top}\mathbf{x}_k$:
    \begin{align*}
        |\gI_{1}| &= \sum_{k \in \gI_1} \min\left\{1, ||\mathbf{x}_k/\bar{\sigma}_k||^2_{\mathbf{Z}_k^{-1}}\mathbbm{1}_{\gX_{\on}} \right\}\\
        &\leq \sum_{k=1}^K \min\left\{1, ||\mathbf{x}_k/\bar{\sigma}_k||^2_{\mathbf{Z}_k^{-1}}\mathbbm{1}_{\gX_{\on}} \right\} \\
        &\leq \sum_{k=1}^K \min\left\{1, \bar{\sigma}_k^{-2}\mathbf{x}_k^{\top}\mathbf{Z}_k^{-1}\mathbf{x}_k\mathbbm{1}_{\gX_{\on}} \right\}\\
        &= \sum_{k=1}^K \min\left\{1,  (U_{\on}U_{\on}^{\top}\mathbf{x}_k)^{\top}\mathbf{Z}_k^{-1}(U_{\on}U_{\on}^{\top}\mathbf{x}_k)\mathbbm{1}_{\gX_{\on}}\right\}\\
        &= \sum_{k=1}^K \min\left\{1,   \mathbf{x}_k^{\top}U_{\on}U_{\on}^{\top}\mathbf{Z}_k^{-1}U_{\on}U_{\on}^{\top}\mathbf{x}_k\mathbbm{1}_{\gX_{\on}}\right\}\\
        &=  \sum_{k=1}^K \min\left\{1,    \mathbf{x}_k^{\top}U_{\on}U_{\on}^{\top}\left(\sum_{n=1}^{k}\mathbbm{1}_{\gX_{\on}}\bar{\sigma}_n^{-2}\mathbf{x}_n\mathbf{x}_n^{\top} + \lambda\mathbf{I}_{d}\right)^{-1} U_{\on}U_{\on}^{\top}\mathbf{x}_k\mathbbm{1}_{\gX_{\on}}\right\}.
    \end{align*}

    By Lemma \ref{lem:hole-digging}, we can take the $U_{\on}$ inside the inverse and conclude that
    \begin{align*}
        &\sum_{k=1}^K \min\left\{1,    \mathbf{x}_k^{\top}U_{\on}U_{\on}^{\top}\left(\sum_{n=1}^{k}\mathbbm{1}_{\gX_{\on}}\bar{\sigma}_n^{-2}\mathbf{x}_n\mathbf{x}_n^{\top} + \lambda\mathbf{I}_{d}\right)^{-1} U_{\on}U_{\on}^{\top}\mathbf{x}_k\mathbbm{1}_{\gX_{\on}}\right\}\\
        & = \sum_{k=1}^K \min\left\{1,    \mathbf{x}_k^{\top}U_{\on}\left(\sum_{n=1}^{k}\mathbbm{1}_{\gX_{\on}}\bar{\sigma}_n^{-2}U_{\on}^{\top}\mathbf{x}_n\mathbf{x}_n^{\top}U_{\on} + \lambda\mathbf{I}_{d_{\on}}\right)^{-1} U_{\on}^{\top}\mathbf{x}_k\mathbbm{1}_{\gX_{\on}}\right\}.
    \end{align*}
    Intuitively, this is because all the $\mathbbm{1}_{\gX_{\on}}U_{\on}^{\top}\mathbf{x}_{n}$ and $\mathbbm{1}_{\gX_{\on}}\mathbf{x}_{n}$ are both in $\Phi_{\on}$, and in that case the projection is just the identity.

  Writing $\mathbf{u}_n = U_{\on}^{\top}\mathbf{x}_n$, and invoking Lemma D.5 of \cite{zhou2022computationally} (which is a restatement of Lemma 11 of \cite{abbas2011improved}) and the fact that $\left\|\mathbf{x}_k / \bar{\sigma}_k\right\|_2 \leq L / \alpha$, it holds that
    \begin{align*}
        &\sum_{k=1}^K \min\left\{1,    \mathbf{x}_k^{\top}U_{\on}\left(\sum_{n=1}^{k}\mathbbm{1}_{\gX_{\on}}\bar{\sigma}^{-2}U_{\on}^{\top}\mathbf{x}_n\mathbf{x}_n^{\top}U_{\on} + \lambda\mathbf{I}_{d_{\on}}\right)^{-1} U_{\on}^{\top}\mathbf{x}_k\mathbbm{1}_{\gX_{\on}}\right\}\\
        &\sum_{k=1}^K \min\left\{1,  \mathbf{u}_k^{\top}\left(\sum_{n=1}^{k}\mathbbm{1}_{\gX_{\on}}\bar{\sigma}^{-2}\mathbf{u}_n\mathbf{u}_n^{\top} + \lambda\mathbf{I}_{d_{\on}}\right)^{-1} \mathbf{u}_k\mathbbm{1}_{\gX_{\on}}\right\}\\
        &\leq 2d_{\on}\iota,
    \end{align*}
    as desired, and conclude that
    $|\gI_{1}| \leq 2d_{\on}\iota.$

    The rest of the proof follows \cite{zhou2022computationally} more closely. By the same argument that \cite{zhou2022computationally} use, 
    \begin{align*}
    \sum_{k \in[K]} \min \left\{1, \beta_k\left\|\mathbf{x}_k\right\|_{\mathbf{z}_k^{-1}}\mathbbm{1}_{\gX_{\on}}\right\} \leq 2 d_{\on} \iota+\sum_{k \in \mathcal{I}_2} \beta_k \bar{\sigma}_k\left\|\mathbf{x}_k / \bar{\sigma}_k\right\|_{\mathbf{z}_k^{-1}}\mathbbm{1}_{\gX_{\on}}.
    \end{align*}

    Decompose $\mathcal{I}_2=\mathcal{J}_1 \cup \mathcal{J}_2$, where
    $$
    \mathcal{J}_1=\left\{k \in \mathcal{I}_2: \bar{\sigma}_k=\sigma_k \cup \bar{\sigma}_k=\alpha\right\}, \mathcal{J}_2=\left\{k \in \mathcal{I}_2: \bar{\sigma}_k=\gamma \sqrt{\left\|\mathbf{x}_k\right\|_{\mathbf{z}_k^{-1}}}\mathbbm{1}_{\gX_{\on}}\right\} .
    $$

    Similar to \cite{zhou2022computationally},
    $$
    \begin{aligned}
    \sum_{k \in \mathcal{J}_1} \beta_k \bar{\sigma}_k\left\|\mathbf{x}_k / \bar{\sigma}_k\right\|_{\mathbf{z}_k^{-1}}\mathbbm{1}_{\gX_{\on}} & \leq \sum_{k \in \mathcal{J}_1} \beta_k\left(\sigma_k+\alpha\right) \mathbbm{1}_{\gX_{\on}}\min \left\{1,\left\|\mathbf{x}_k / \bar{\sigma}_k\right\|_{\mathbf{z}_k^{-1}}\mathbbm{1}_{\gX_{\on}}\right\} \\
    & \leq \sum_{k=1}^K \beta_k\left(\sigma_k+\alpha\right) \min \left\{1,\left\|\mathbf{x}_k / \bar{\sigma}_k\right\|_{\mathbf{z}_k^{-1}}\mathbbm{1}_{\gX_{\on}}\right\} \\
    & \leq \sqrt{2 \sum_{k=1}^K\left(\sigma_k^2+\alpha^2\right) \beta_k^2} \sqrt{\sum_{k=1}^K \min \left\{1,\left\|\mathbf{x}_k / \bar{\sigma}_k\right\|_{\mathbf{z}_k^{-1}}\mathbbm{1}_{\gX_{\on}}\right\}^2} \\
    & \leq 2 \sqrt{\sum_{k=1}^K \beta_k^2\left(\sigma_k^2+\alpha^2\right)} \sqrt{d_{\on} \iota},
    \end{aligned}
    $$
    and as for $k\in \gJ_2$ we have that $\bar{\sigma}_k = \gamma^2\left\|\mathbf{x}_k / \bar{\sigma}_k\right\|_{\mathbf{Z}_k^{-1}}\mathbbm{1}_{\gX_{\on}}$,
    \begin{align*}
        \sum_{k \in \mathcal{J}_2} \beta_k \bar{\sigma}_k\left\|\mathbf{x}_k / \bar{\sigma}_k\right\|_{\mathbf{Z}_k^{-1}}\mathbbm{1}_{\gX_{\on}}&=\gamma^2 \cdot \sum_{k \in \mathcal{J}_1} \beta_k\left\|\mathbf{x}_k / \bar{\sigma}_k\right\|_{\mathbf{Z}_k^{-1}}^2\mathbbm{1}_{\gX_{\on}}\\
        &=\gamma^2 \cdot \sum_{k=1}^K \beta_k \min \left\{1,\left\|\mathbf{x}_k / \bar{\sigma}_k\right\|_{\mathbf{Z}_k^{-1}}^2\mathbbm{1}_{\gX_{\on}}\right\} \leq 2 \max _{k \in[K]} \beta_k \gamma^2 d_{\on} \iota.
    \end{align*}
    
    Therefore,
    $$
    \sum_{k=1}^K \min \left\{1, \beta_k\left\|\mathbf{x}_k\right\|_{\mathbf{z}_k^{-1}}\mathbbm{1}_{\gX_{\on}}\right\} \leq 2 d_{\on} \iota+2 \max _{k \in[K]} \beta_k \gamma^2 d_{\on} \iota+2 \sqrt{d_{\on} \iota} \sqrt{\sum_{k=1}^K \beta_k^2\left(\sigma_k^2+\alpha^2\right)}.
    $$
\end{proof}

\begin{lem}[Modified Version of Theorem 4.3,  \cite{zhou2022computationally}]
    \label{lem:4.1}
    Let $\left\{\mathcal{G}_n\right\}_{n=1}^{N}$ be a filtration, and $\left\{\mathbf{x}_n, \eta_n\right\}_{n=1}^N$ be a stochastic process such that $\mathbf{x}_n \in \mathbb{R}^d$ is $\mathcal{G}_n$-measurable and $\eta_n \in \mathbb{R}$ is $\mathcal{G}_{n+1}$-measurable. Let $L, \sigma, \lambda, \epsilon>0, \boldsymbol{\mu}^* \in \mathbb{R}^d$. Arrange the datapoints from the offline and online samples as follows, $1,...,N_{\off}, N_{\off}+1,...,N_{\off}+N_{\on}$. For $n =1,...,N$, let $y_n=\left\langle\boldsymbol{\mu}^*, \mathbf{x}_n\right\rangle+\eta_n$ and suppose that $\eta_n, \mathbf{x}_n$ also satisfy
$$
\mathbb{E}\left[\eta_n \mid \mathcal{G}_n\right]=0, \mathbb{E}\left[\eta_n^2 \mid \mathcal{G}_n\right] \leq \sigma^2,\left|\eta_n\right| \leq R,\left\|\mathbf{x}_n\right\|_2 \leq L.
$$

For $n =1,...,N$, let $\mathbf{Z}_n=\lambda \mathbf{I}+\sum_{i=1}^n \mathbf{x}_i \mathbf{x}_i^{\top}, \mathbf{b}_n=\sum_{i=1}^n y_i \mathbf{x}_i, \boldsymbol{\mu}_n=\mathbf{Z}_n^{-1} \mathbf{b}_n$, and

\begin{align*}
\beta_n= & 12 \sqrt{\sigma^2 d \log \left(1+n L^2 /(d \lambda)\right) \log \left(32(\log (R / \epsilon)+1) n^2 / \delta\right)} \\
& +24 \log \left(32(\log (R / \epsilon)+1) n^2 / \delta\right) \max _{1 \leq i \leq n}\left\{\left|\eta_i\right| \min \left\{1,\left\|\mathbf{x}_i\right\|_{\mathbf{z}_{i-1}^{-1}}\right\}\right\}\\
&+6 \log \left(32(\log (R / \epsilon)+1) n^2 / \delta\right) \epsilon .
\end{align*}

Then, for any $0<\delta<1$, we have with probability at least $1-\delta$ that,
$$
\forall n =1,...,N,\left\|\sum_{i=1}^n \mathbf{x}_i \eta_i\right\|_{\mathbf{z}_n^{-1}} \leq \beta_n,\left\|\boldsymbol{\mu}_n-\boldsymbol{\mu}^*\right\|_{\mathbf{z}_n} \leq \beta_n+\sqrt{\lambda}\left\|\boldsymbol{\mu}^*\right\|_2
$$
\end{lem}

\begin{proof}
    The proof is merely a small wrapper over Theorem 4.3 of \cite{zhou2022computationally}, where we adapt this to our setting in the same way that \cite{tan2024natural} do in Lemma 1 of their paper. That is, we pre-append the offline data to the online data, and generate the $\bf{Z}_n, \bf{b}_n, \bf{\mu}_n, \beta_n$ accordingly. 

    As in Lemma 1 of \cite{tan2024natural}, let $N=N_{\off}+N_{\on}$. Order the $N_{\off}$ offline episodes arbitrarily, to form episodes $1,...,N_{\off}$, and then begin the online episodes from episode $N_{\off}+1,...,N$. Then, we can directly apply Theorem 4.3 of \cite{zhou2022computationally} to recover the desired result. 
\end{proof}

\begin{lem}
  Suppose that $W = \mathbb{R}^m$ and $V = \mathbb{R}^n$, where $n<m$. Let ${\bm U}: W\mapsto V$ be a linear transformation and that $S = ({\bm U}^{\top} {\bm U}) W$. As ${\bm v}, {\bm v}_1,\ldots, {\bm v}_n\in S$, we have
  $${\bm v}^{\top} {\bm U}^{\top} {\bm U}\Big(\sum_{j=1}^{k}{\bm v}_i {\bm v}_i^{\top} + \lambda {\bm I}_{m}\Big)^{-1} {\bm U}^{\top} {\bm U} {\bm v} = {\bm v}^{\top} {\bm U}^{\top} \Big(\sum_{j=1}^{k}{\bm U} {\bm v}_i {\bm v}_i^{\top} {\bm U}^{\top} + \lambda {\bm I}_{n}\Big)^{-1} {\bm U} {\bm v} $$
  \label{lem:hole-digging}
\end{lem}
\begin{proof}
    For projection matrix ${\bm U}$, there exists orthogonal matrix ${\bm Q}\in\mathbb{R}^{m\times m}$ and diagonal matrix ${\bm D} = ({\bm I}_{n}, {\bm 0}_{n\times (m-n)})$ such that ${\bm U}= {\bm {DQ}}$. We further define ${\bm u} = {\bm U} {\bm v} $, $\tilde{{\bm v}} = {\bm Q} {\bm v}$, ${\bm u}_i = {\bm U} {\bm v}_i$ and $\tilde{{\bm v}}_i ={\bm Q} {\bm v}_i$ for $1\leq i\leq n$. Then, we note that as ${\bm v}\in S$, we have ${\bm v} = {\bm {U^{\top} U}} {\bm v} = {\bm {Q^{\top} \Lambda Q}} {\bm v}$, where ${\bm \Lambda} = \text{diag}({\bm I}_{n}, {\bm 0}_{m-n})$, which is equivalent to $\tilde{{\bm v}} = \Lambda \tilde{{\bm v}}$. As a result, we may conclude that $\tilde{{\bm v}}^{\top} = ({\bm u}^{\top}, {\bm 0}_{m-n})$.

    Therefore, with a direct calculation, one will see that
    \begin{align}
        \Big(\sum_{j=1}^{k}{\bm v}_i {\bm v}_i^{\top} + \lambda {\bm I}_{m}\Big)^{-1}
       & = \Big(\sum_{j=1}^{k} Q^{\top} \tilde{{\bm v}}_i\tilde{{\bm v}}_i^{\top} Q+ \lambda {\bm I}_{m}\Big)^{-1}\nonumber\\
       & = {\bm Q}^{\top} \Big(\sum_{j=1}^{k}\tilde{{\bm v}}_i \tilde{{\bm v}}_i^{\top} + \lambda {\bm I}_{m}\Big)^{-1} {\bm Q}\nonumber\\
     &  ={\bm Q}^{\top} \begin{pmatrix}
          \sum_{i=1}^{k} {\bm u}_i {\bm u}_i^{\top} + \lambda{\bm I}_n & {\bm 0} \\
          {\bm 0} & \lambda{\bm I}_{m-n}
      \end{pmatrix}^{-1} {\bm Q} \nonumber\\
      & = {\bm Q}^{\top}\begin{pmatrix}
          \Big(\sum_{j=1}^{k}{\bm u}_i {\bm u}_i^{\top} + \lambda {\bm I}_{n}\Big)^{-1} & {\bm 0} \\
          {\bm 0} & \lambda^{-1}{\bm I}_{m-n}
      \end{pmatrix}
       {\bm Q}. \nonumber
    \end{align}
    This will establish
    our desired conclusion
    \begin{align}
      \text{LHS} & = {\bm v}^{\top} \Big(\sum_{j=1}^{k}{\bm v}_i {\bm v}_i^{\top} + \lambda {\bm I}_{m}\Big)^{-1}  {\bm v}  = \tilde{{\bm v}}^{\top} \begin{pmatrix}
          \Big(\sum_{j=1}^{k}{\bm u}_i {\bm u}_i^{\top} + \lambda {\bm I}_{n}\Big)^{-1} & {\bm 0} \\
          {\bm 0} & \lambda^{-1}{\bm I}_{m-n}
      \end{pmatrix} \tilde{{\bm v}} \nonumber \\
      & = {\bm u}^{\top}  \Big(\sum_{j=1}^{k}{\bm u}_i {\bm u}_i^{\top} + \lambda {\bm I}_{n}\Big)^{-1} {\bm u}  = \text{RHS}.\nonumber
    \end{align}
\end{proof}


\end{document}